\documentclass{article}

\usepackage[final, nonatbib]{neurips_2019}

\usepackage[utf8]{inputenc}
\usepackage[T1]{fontenc}
\usepackage{hyperref}
\usepackage{url}
\usepackage{booktabs}
\usepackage{amsfonts}
\usepackage{nicefrac}
\usepackage{microtype}
\usepackage{caption}
\usepackage{subfig}
\usepackage{xr}

\usepackage{xargs,amsmath,bbm}
\usepackage{stmaryrd}
\usepackage[normalem]{ulem}
\usepackage{graphicx}

\SetSymbolFont{stmry}{bold}{U}{stmry}{m}{n}
\let\zz\[\let\zzz\]

\usepackage{enumitem}
\usepackage{ifthen}
\usepackage{xargs}
\usepackage{nicefrac}
\usepackage{amssymb,mathrsfs} 
\usepackage{yfonts}
\usepackage{upgreek}
\usepackage{wrapfig}
\usepackage{aliascnt}
\usepackage{cleveref}

\let\[\zz\let\]\zzz

\usepackage{amsthm}

\makeatletter
\newtheorem{theorem}{Theorem}
\crefname{theorem}{theorem}{Theorems}
\Crefname{Theorem}{Theorem}{Theorems}

\newtheorem{proposition}{Proposition}
\crefname{proposition}{proposition}{propositions}
\Crefname{Proposition}{Proposition}{Propositions}

\newtheorem{lemma}{Lemma}
\crefname{lemma}{lemma}{lemmas}
\Crefname{Lemma}{Lemma}{Lemmas}

\newtheorem{corollary}{Corollary}
\crefname{corollary}{corollary}{corollaries}
\Crefname{Corollary}{Corollary}{Corollaries}

\newtheorem{definition}{Definition}
\crefname{definition}{definition}{definitions}
\Crefname{Definition}{Definition}{Definitions}

\newtheorem{assumption}{\textbf{A}\hspace{-3pt}}
\crefformat{assumption}{{\textbf{A}}#2#1#3}

\crefname{remark}{remark}{remarks}
\Crefname{Remark}{Remark}{Remarks}

\crefname{figure}{figure}{figures}
\Crefname{Figure}{Figure}{Figures}

\newcommand{\Rd}{\mathbb{R}^d}

\newcommand{\ie}{\textit{i.e.}}

\newcommand{\calM}{\mathcal{M}}
\newcommand{\calN}{\mathcal{N}}

\newcommand{\calP}{\mathcal{P}}

\newcommand{\calS}{\mathcal{S}}

\newcommand{\bbN}{\mathbb{N}}

\newcommand{\Esp}{\mathbb{E}}

\newcommand{\rmi}{\mathrm{i}}
\newcommand{\rme}{\mathrm{e}}

\newcommand{\bfI}{\mathbf{I}}

\newcommand{\bfm}{\mathbf{m}}
\newcommand{\bft}{\mathbf{t}}
\newcommand{\bfs}{\mathbf{s}}

\newcommand{\bfSigma}{\mathbf{\Sigma}}

\newcommand{\bfnot}{\mathbf{0}}

\def\msi{\mathsf{I}}
\def\msa{\mathsf{A}}
\def\msd{\mathsf{X}}
\def\msk{\mathsf{K}}
\def\mss{\mathsf{S}}

\def\msb{\mathsf{B}} 

\def\msd{\mathsf{D}}
\def\mse{\mathsf{E}}

\def\msi{\mathsf{I}}

\def\mso{\mathsf{O}}
\def\msu{\mathsf{U}}
\def\msv{\mathsf{V}}

\def\msz{\mathsf{Z}}

\def\msy{\mathsf{Y}}

\def\mcl{\mathcal{L}}
\def\mcs{\mathcal{S}}

\newcommand{\mcb}[1]{\mathcal{B}(#1)}

\def\mcy{\mathcal{Y}}

\def\mce{\mathcal{E}}
\def\mcf{\mathcal{F}}

\def\mcp{\mathcal{P}}

\def\rset{\mathbb{R}}

\def\nset{\mathbb{N}}
\def\nsets{\mathbb{N}^*}

\def\rmd{\mathrm{d}}

\def\rme{\mathrm{e}}

\newcommand{\abs}[1]{\left\vert #1 \right\vert}

\newcommandx{\Vnorm}[2][1=V]{\| #2 \|_{#1}}
\newcommandx{\normpi}[2][2=2]{\left\Vert  #1 \right\Vert_{#2}}
\newcommandx{\normH}[2][2=2]{\left\Vert  #1 \right\Vert}
\newcommandx{\normHLigne}[2][2=2]{\Vert  #1 \Vert}
\newcommandx{\normHLine}[2][2=2]{\Vert  #1 \Vert}
\newcommandx{\normmu}[2][2=2]{\left\Vert  #1 \right\Vert_{#2}}
\newcommandx{\normopmu}[2][2=2]{\left\vvvert  #1 \right\vvvert_{#2}}
\newcommandx{\normopH}[2][2=2]{\left\vvvert  #1 \right\vvvert}
\newcommandx{\normop}[2][2=2]{\left\vvvert  #1 \right\vvvert}
\newcommand{\ps}[2]{\left\langle#1,#2 \right\rangle}

\newcommandx{\normpiLine}[2][2=2]{\Vert  #1 \Vert_{#2}}
\newcommandx{\normmuLine}[2][2=2]{\Vert  #1 \Vert_{#2}}
\newcommandx{\normopmuLine}[2][2=2]{\vvvert  #1 \vvvert_{#2}}
\newcommandx{\normopHLine}[2][2=2]{\vvvert  #1 \vvvert}
\newcommandx{\normopLine}[2][2=2]{\vvvert  #1 \vvvert}

\newcommandx{\VnormEq}[2][1=V]{\left\| #2 \right\|_{#1}}
\newcommandx{\norm}[2][1=]{\ifthenelse{\equal{#1}{}}{\left\Vert #2 \right\Vert}{\left\Vert #2 \right\Vert^{#1}}}
\newcommandx{\normLigne}[2][1=]{\ifthenelse{\equal{#1}{}}{\Vert #2 \Vert}{\Vert #2\Vert^{#1}}}
\newcommandx{\norminf}[2][1=]{\ifthenelse{\equal{#1}{}}{\left\Vert #2 \right\Vert_{\infty}}{\left\Vert #2 \right\Vert^{#1}_{\infty}}}

\newcommand{\parenthese}[1]{\left(#1 \right)}

\newcommand{\parentheseDeux}[1]{\left[ #1 \right]}
\newcommand{\defEns}[1]{\left\lbrace #1 \right\rbrace }

\newcommand{\PE}{\mathbb{E}}
\newcommand{\PP}{\mathbb{P}}
\newcommand{\expe}[1]{\PE \left[ #1 \right]}
\newcommand{\expeY}[1]{\PE \left[ #1 \middle| Y_{1:n} \right]}

\newcommandx{\expeE}[2][2=]{\mathbb{E}^{#2}\left[ #1 \right]}
\newcommand{\expeLigne}[1]{\PE [ #1 ]}
\newcommand{\expeYLigne}[1]{\PE [ #1  | Y_{1:n} ]}

\newcommand{\plusinfty}{+\infty}
\def\ie{\textit{i.e.}}
\def\eqsp{\;}

\newcommand{\ocint}[1]{\left(#1\right]}

\newcommand{\ccint}[1]{\left[#1\right]}

\newcommandx\sequence[3][2=,3=]
{\ifthenelse{\equal{#3}{}}{\ensuremath{\{ #1_{#2}\}}}{\ensuremath{\{ #1_{#2}, \eqsp #2 \in #3 \}}}}
\newcommandx\sequenceD[3][2=,3=]
{\ifthenelse{\equal{#3}{}}{\ensuremath{\{ #1_{#2}\}}}{\ensuremath{( #1)_{ #2 \in #3} }}}

\newcommandx{\sequencen}[2][2=n\in\nset]{\ensuremath{( #1)_{#2}}}
\newcommandx{\sequencek}[2][2=k\in\nset]{\ensuremath{( #1)_{#2}}}
\newcommandx{\sequencens}[2][2=n\in\nsets]{\ensuremath{( #1)_{#2}}}
\newcommandx{\sequenceks}[2][2=k\in\nsets]{\ensuremath{( #1)_{#2}}}
\newcommandx\sequenceDouble[4][3=,4=]
{\ifthenelse{\equal{#3}{}}{\ensuremath{\{ (#1_{#3},#2_{#3}) \}}}{\ensuremath{\{  (#1_{#3},#2_{#3}), \eqsp #3 \in #4 \}}}}
\newcommandx{\sequencenDouble}[3][3=n\in\nset]{\ensuremath{\{ (#1_{n},#2_{n}), \eqsp #3 \}}}

\def\iid{i.i.d.}

\def\Idd{\operatorname{I}_d}
\newcommand{\ensemble}[2]{\left\{#1\,:\eqsp #2\right\}}

\newcommand{\set}[2]{\ensemble{#1}{#2}}

\def\card{\operatorname{card}}

\def\sphere{\mathbb{S}}
\def\sphereD{\mathbb{S}^{d-1}}

\newcommand{\1}{\mathbbm{1}}

\def\distance{\ell}
\newcommandx{\wasserstein}[3][1=\distance,3=]{\mathbf{W}_{#1}^{#3}\left(#2\right)}
\newcommandx{\wassersteinLigne}[3][1=\distance,3=]{\mathbf{W}_{#1}^{#3}(#2)}
\newcommandx{\wassersteinD}[1][1=\distance]{\mathbf{W}_{#1}}
\newcommandx{\wassersteinDLigne}[1][1=\distance]{\mathbf{W}_{#1}}

\newcommandx{\swasserstein}[3][1=\distance,3=]{\mathbf{SW}_{#1}^{#3}\left(#2\right)}
\newcommandx{\swassersteinLigne}[3][1=\distance,3=]{\mathbf{SW}_{#1}^{#3}(#2)}
\newcommandx{\swassersteinD}[1][1=\distance]{\mathbf{SW}_{#1}}
\newcommandx{\swassersteinDLigne}[1][1=\distance]{\mathbf{SW}_{#1}}
\newcommandx{\hatswassersteinD}[1][1=\distance]{\widehat{\mathbf{SW}}_{#1}}

\def\dist{\mathbf{d}}

\def\Leb{\mathrm{Leb}}

\def\argmin{\mathrm{argmin}}

\def\us{u}
\def\uss{u^{\star}}
\def\usss{u^{\star}_{\sharp}}

\def\unifS{\boldsymbol{\sigma}}
\def\lsc{l.s.c.}
\def\wc{\xrightarrow{w}}
\def\ec{\xrightarrow{e}}

\def\proj{\operatorname{proj}}
\def\projD{\proj(\msd)}
\def\ellstable{\mce \alpha \mcs_c}

\newcommand\fraca[2]{(#1/#2)}

\def\bfepsilon{\boldsymbol{\epsilon}}
\def\hmu{\hat{\mu}}
\def\hnu{\hat{\nu}}
\def\htheta{\hat{\theta}}

\title{Asymptotic Guarantees for Learning Generative Models with the Sliced-Wasserstein Distance}

\author{%
    Kimia Nadjahi$^{1}$, Alain Durmus$^{2}$, Umut \c{S}im\c{s}ekli$^{1,3}$, Roland Badeau$^{1}$  \\
    1: LTCI, T\'{e}l\'{e}com Paris, Institut Polytechnique de Paris, France\\
    2: CMLA, ENS Cachan, CNRS, Universit\'e Paris-Saclay, France \\
    3: Department of Statistics, University of Oxford, UK \\
    \texttt{\{kimia.nadjahi, umut.simsekli, roland.badeau\}@telecom-paris.fr} \\
    \texttt{alain.durmus@cmla.ens-cachan.fr}
}

\begin{document}

\maketitle

\begin{abstract}
Minimum expected distance estimation (MEDE) algorithms have been widely used for probabilistic models with intractable likelihood functions and they have become increasingly popular due to their use in implicit generative modeling (e.g.\ Wasserstein generative adversarial networks, Wasserstein autoencoders). Emerging from computational optimal transport, the Sliced-Wasserstein (SW) distance has become a popular choice in MEDE thanks to its simplicity and computational benefits. While several studies have reported empirical success on generative modeling with SW, the theoretical properties of such estimators have not yet been established. In this study, we investigate the asymptotic properties of estimators that are obtained by minimizing SW. We first show that convergence in SW implies weak convergence of probability measures in general Wasserstein spaces. Then we show that estimators obtained by minimizing SW (and also an approximate version of SW) are asymptotically consistent. We finally prove a central limit theorem, which characterizes the asymptotic distribution of the estimators and establish a convergence rate of $\sqrt{n}$, where $n$ denotes the number of observed data points. We illustrate the validity of our theory on both synthetic data and neural networks. 
\end{abstract}

\section{Introduction}

Minimum distance estimation (MDE) is a generalization of maximum-likelihood inference, where the goal is to minimize a distance between the empirical distribution of a set of independent and identically distributed (i.i.d.) observations $Y_{1:n} =(Y_1,\dots,Y_n)$ and a family of distributions indexed by a parameter $\theta$. The problem is formally defined as follows \cite{Wolfowitz1957, Basu2011Statistical}:
\begin{align} 
  \label{eq:mde}
    \hat{\theta}_n = \argmin_{\theta \in \Theta}\ 
    \mathbf{D}(\hat{\mu}_n, \mu_\theta) \eqsp,
\end{align}
where $\mathbf{D}$ denotes a distance (or a divergence in general) between probability measures, $\mu_\theta$ denotes a probability measure indexed by $\theta$, $\Theta$ denotes the parameter space, and
\begin{equation}
  \label{eq:def_empirical_measure_Y}
  \hat{\mu}_n= \frac{1}{n}  \sum\nolimits_{i=1}^n \updelta_{Y_i}
\end{equation}
denotes the empirical measure of $Y_{1:n}$, with $\updelta_Y$ being the Dirac distribution with mass on the point $Y$. When $\mathbf{D}$ is chosen as the Kullback-Leibler divergence, this formulation coincides with the maximum likelihood estimation (MLE) \cite{Basu2011Statistical}.

While MDE provides a fruitful framework for statistical inference, when working with generative models, solving the optimization problem in \eqref{eq:mde} might be intractable since it might be impossible to evaluate the probability density function associated with $\mu_\theta$.
Nevertheless, in various settings, even if the density is not available, one can still generate samples from the distribution $\mu_\theta$, and such samples turn out to be useful for making inference. More precisely, under such settings, a natural alternative to \eqref{eq:mde} is the minimum \emph{expected} distance estimator, which is defined as follows \cite{Bernton2019}: 
\begin{equation} \label{eq:mede}
    \hat{\theta}_{n, m} = \argmin_{\theta \in \Theta}\ \expe{ \mathbf{D}(\hat{\mu}_n, \hat{\mu}_{\theta,m}) \middle| Y_{1:n} } \eqsp.
\end{equation}
Here,
\begin{equation}
  \label{eq:def_empirical_measure_Z}
  \hat{\mu}_{\theta, m} = \frac{1}{m} \sum\nolimits_{i=1}^m \updelta_{Z_i} 
\end{equation}
denotes the empirical distribution of $Z_{1:m}$, that is a sequence of i.i.d.\ random variables with distribution $\mu_\theta$. %
This algorithmic framework has computationally favorable properties since one can replace the expectation with a simple Monte Carlo average in practical applications. 

In the context of MDE, distances that are based on optimal transport (OT) have become increasingly popular due to their computational and theoretical properties \cite{arjovsky2017wasserstein,tolstikhin2017wasserstein,genevay2017learning,patrini2018sinkhorn,adler2018banach}. For instance, if we replace the distance $\mathbf{D}$ in \eqref{eq:mede} with the Wasserstein distance (defined in \Cref{sec:prel-techn-backgr} below), we obtain the minimum expected Wasserstein estimator \cite{Bernton2019}. In the classical statistical inference setting, the typical use of such an estimator is to infer the parameters of a measure whose density does not admit an analytical closed-form formula \cite{Basu2011Statistical}. On the other hand, in the implicit generative modeling (IGM) setting, this estimator forms the basis of two popular IGM strategies: Wasserstein generative adversarial networks (GAN) \cite{arjovsky2017wasserstein} and Wasserstein variational auto-encoders (VAE) \cite{tolstikhin2017wasserstein} (cf.\ \cite{genevay2017gan} for their relation). The goal of these two methods is to find the best parametric \emph{transport map} $T_\theta$, such that $T_\theta$ transforms a simple distribution $\mu$ (e.g.\ standard Gaussian or uniform) to a potentially complicated data distribution $\hat{\mu}_n$ by minimizing the Wasserstein distance between the transported distribution $\mu_\theta = T_{\theta\sharp}\mu$ and $\hat{\mu}_n$, where $\sharp$ denotes the push-forward operator, to be defined in the next section. In practice, $\theta$ is typically chosen as a neural network,
for which it is often impossible to evaluate the induced density $\mu_\theta$. However, one can easily generate samples from $\mu_\theta$ by first generating a sample from $\mu$ and then applying $T_\theta$ to that sample, making minimum expected distance estimation \eqref{eq:mede} feasible for this setting. Motivated by its practical success, the theoretical properties of this estimator have been recently taken under investigation \cite{bousquet2017optimal,liu2017approximation} and very recently Bernton et al.\ \cite{Bernton2019} have established the consistency (for the general setting) and the asymptotic distribution (for one dimensional setting) of this estimator.
 
Even though estimation with the Wasserstein distance has served as a fertile ground for many generative modeling applications, except for the case when the measures are supported on $\rset^1$, the computational complexity of minimum Wasserstein estimators rapidly becomes excessive with the increasing problem dimension, and developing accurate and efficient approximations is a highly non-trivial task. Therefore, there have been several attempts to use more practical alternatives to the Wasserstein distance \cite{cuturi2013sinkhorn,genevay2017learning}. In this context, the Sliced-Wasserstein (SW) distance \cite{rabin:et:al:2011,Bonnotte2013,bonneel2015sliced} has been an increasingly popular alternative to the Wasserstein distance, which is defined as an average of \emph{one-dimensional} Wasserstein distances, which allows it to be computed in an efficient manner. 

While several studies have reported empirical success on generative modeling with SW \cite{Deshpande2018,Kolouri2018Sliced,csimcsekli2018sliced,wu2017sliced}, the theoretical properties of such estimators have not yet been fully established. Bonnotte \cite{Bonnotte2013} proved that SW is a proper metric, and in \emph{compact} domains SW is equivalent to the Wasserstein distance, hence convergence in SW implies weak convergence in compact domains. \cite{Bonnotte2013} also analyzed the gradient flows based on SW, which then served as a basis for a recently proposed IGM algorithm \cite{csimcsekli2018sliced}. Finally, recent studies \cite{Deshpande2018,deshpande2019max} investigated the sample complexity of SW and established bounds for the SW distance between two measures and their empirical instantiations. 

In this paper, we investigate the asymptotic properties of estimators given in \eqref{eq:mde} and \eqref{eq:mede} when $\mathbf{D}$ is replaced with the SW distance. We first prove that convergence in SW implies weak convergence of probability measures defined on general domains, which generalizes the results given in \cite{Bonnotte2013}. Then, by using similar techniques to the ones given in \cite{Bernton2019}, we show that the estimators defined by \eqref{eq:mde} and \eqref{eq:mede} are consistent, meaning that as the number of observations $n$ increases the estimates will get closer to the data-generating parameters. We finally prove a central limit theorem (CLT) in the multidimensional setting, which characterizes the asymptotic distribution of these estimators and establishes a convergence rate of $\sqrt{n}$. The CLT that we prove is stronger than the one given in \cite{Bernton2019} in the sense that it is not restricted to the one-dimensional setting as opposed to \cite{Bernton2019}.

We support our theory with experiments that are conducted on both synthetic and real data. We first consider a more classical statistical inference setting, where we consider a Gaussian model and a multidimensional $\alpha$-stable model whose density is not available in closed-form. In both models, the experiments validate our consistency and CLT results. We further observe that, especially for high-dimensional problems, the estimators obtained by minimizing SW have significantly better computational properties when compared to the ones obtained by minimizing the Wasserstein distance, as expected. In the IGM setting, we consider the neural network-based generative modeling algorithm proposed in \cite{Deshpande2018} and show that our results also hold in the real data setting as well.  

\section{Preliminaries and Technical Background}
\label{sec:prel-techn-backgr}

We consider a probability space $(\Omega, \mathcal{F}, \mathbb{P})$ with associated expectation operator $\mathbb{E}$, on which all the random variables are defined. Let $(Y_k)_{k \in \mathbb{N}}$ be a sequence of random variables associated with observations, where each observation takes value in $\msy \subset \mathbb{R}^{d}$. We assume that these observations are \iid~according to $\mu_\star \in \mcp(\msy)$, where $\mathcal{P}(\msy)$ stands for the set of probability measures on $\msy$. 

A statistical model is a family of distributions on $\msy$ and is denoted by $\mathcal{M} = \{ \mu_\theta \in \mcp(\msy),\ \theta \in \Theta \}$, where $\Theta \subset \mathbb{R}^{d_\theta}$ is the parametric space. In this paper, we focus on parameter inference for purely generative models: for all $\theta \in \Theta$, we can generate \iid~samples $(Z_{k})_{k \in \nsets}\in \msy^{\nsets}$ from $\mu_\theta$, but the associated likelihood is numerically intractable. In the sequel, $(Z_k)_{k \in \nsets}$  denotes an \iid~sequence from $\mu_{\theta}$ with $\theta \in \Theta$, and for any $m \in \nsets$,  $\hat{\mu}_{\theta, m} = \fraca{1}{m} \sum_{i=1}^m \updelta_{Z_i}$ denotes the corresponding empirical distribution. 

Throughout our study, we assume that the following conditions hold: (1) $\msy$, endowed with the Euclidean distance $\rho$, is a Polish space, (2) $\Theta$, endowed with the distance $\rho_\Theta$, is a Polish space, (3) $\Theta$ is a $\sigma$-compact space, \ie~the union of countably many compact subspaces, and (4) parameters are identifiable, \textit{i.e.} $\mu_\theta = \mu_{\theta'}$ implies $\theta = \theta'$. We endow $\mathcal{P}(\msy)$ with the L\'{e}vy-Prokhorov distance $\dist_{\mathcal{P}}$, which metrizes the weak convergence by \cite[Theorem 6.8]{Billingsley1999} since $\msy$ is assumed to be a Polish space. We denote by $\mcy$ the Borel $\sigma$-field of $(\msy,\rho)$. 

\textbf{Wasserstein distance. } 
For $p \geq 1$, we denote by $\calP_p(\msy)$ the set of probability measures on $\msy$ with finite $p$'th moment: $\mathcal{P}_p(\msy) = \set{ \mu \in \mcp(\msy)}{\int_{\msy} \norm{y - y_0 }^p \rmd\mu(y) < \plusinfty, \, \text{ for some $y_0 \in \msy$}}$. The Wasserstein distance of order $p$ between any $\mu, \nu \in \calP_p(\msy)$ is defined by  %
\cite{Villani2008}, 
\begin{equation}
\label{eq:def_wasser}
 \wassersteinD[p]^{p}(\mu, \nu) = \inf_{\gamma \in \Gamma(\mu, \nu)} \defEns{ \int_{\msy \times \msy} \norm{ x - y }^p \rmd\gamma(x,y)}  \eqsp,
 \end{equation}
where $\Gamma(\mu, \nu)$ is the set of probability measures $\gamma$ on $(\msy \times \msy,\mcy \otimes \mcy)$ satisfying $\gamma(\msa \times \msy) = \mu(\msa)$ and $\gamma(\msy\times \msa) = \nu(\msa)$ for any $\msa \in \mcb{\msy}$. The space  $\mathcal{P}_p(\msy)$ endowed with the distance   $\wassersteinD[p]$ is a Polish space by \cite[Theorem 6.18]{Villani2008} since $(\msy, \rho)$ is assumed to be Polish.

The one-dimensional case is a  favorable scenario for which computing the Wasserstein distance of order $p$ between $\mu, \nu \in \calP_p(\rset)$ becomes relatively easy since it has a closed-form formula, given by \cite[Theorem 3.1.2.(a)]{rachev:ruschendorf:1998}:
\begin{align}
  \label{eq:wp1d}
    \wassersteinD[p]^p(\mu, \nu)    = \int_{0}^1 \abs{ F_\mu^{-1}(t) - F_\nu^{-1}(t) }^p \rmd t = \int_{\rset} \abs{ s - F_\nu^{-1}(F_\mu(s))}^p \rmd\mu(s) 
\eqsp,
\end{align}
where $F_\mu$ and $F_\nu$ denote the cumulative distribution functions (CDF) of $\mu$ and $\nu$ respectively, and $F_\mu^{-1}$ and $F_\nu^{-1}$ are the quantile functions of $\mu$ and $\nu$ respectively. %
For empirical distributions, \eqref{eq:wp1d} is calculated by simply sorting the $n$ samples drawn from each distribution and computing the average cost between the sorted samples.

\textbf{Sliced-Wasserstein distance. }
The analytical form of the Wasserstein distance for one-dimensional distributions is an attractive property that gives rise to an alternative metric referred to as the Sliced-Wasserstein (SW) distance \cite{rabin:et:al:2011,bonneel2015sliced}. The idea behind SW is to first, obtain a family of one-dimensional representations for a higher-dimensional probability distribution through linear projections, and then, compute the average of the Wasserstein distance between these one-dimensional representations. 

More formally, let $\sphere^{d-1} = \set{\us \in \rset^d}{\norm{\us} = 1}$ be the $d$-dimensional unit sphere, and denote by $\ps{\cdot}{\cdot}$ the Euclidean inner-product. For any $\us \in \sphereD$, we define $\uss$ the linear form associated with $\us$ for any $y \in \msy$ by $\uss(y) = \ps{u}{y}$.  The Sliced-Wasserstein distance of order $p$ is defined for any $\mu,\nu \in \mathcal{P}_p(\msy)$ as,
\begin{equation}
\label{eq:def_sliced_wasser}
  \swassersteinD[p]^{p}(\mu, \nu) = \int_{\sphere^{d-1}} \wassersteinD[p]^p(\uss_{\sharp} \mu, \uss_{\sharp} \nu) \rmd\unifS(\us) 
\end{equation}
where $\unifS$ is the uniform distribution on $\sphere^{d-1}$ and for any measurable function $f :\msy \to \rset$ and $\zeta \in \mcp(\msy)$,  $f_{\sharp}\zeta$ is the push-forward measure of $\zeta$ by $f$, \ie~for any $\msa \in \mcb{\rset}$, $f_{\sharp}\zeta(\msa) = \zeta(f^{-1}(\msa))$ where $f^{-1}(\msa) = \{y \in \msy \, : \, f(y) \in \msa\}$.

$\swassersteinD[p]$ is a distance on $\calP_p(\msy)$ \cite{Bonnotte2013} and has important practical implications: in practice, the integration in \eqref{eq:def_sliced_wasser} is approximated using a Monte Carlo scheme that randomly draws a finite set of samples from $\unifS$ on $\sphere^{d-1}$ and replaces the integral with a finite-sample average. Therefore, the evaluation of the SW distance between $\mu, \nu \in \calP_p(\msy)$ has significantly lower computational requirements than the Wasserstein distance, since it consists in solving several one-dimensional optimal transport problems, which have closed-form solutions. 

\section{Asymptotic Guarantees for Minimum Sliced-Wasserstein Estimators}

We define the \emph{minimum Sliced-Wasserstein estimator} (MSWE) \emph{of order $p$} as the
estimator obtained by plugging $\swassersteinD[p]$ in place of $\mathbf{D}$ in \eqref{eq:mde}. Similarly, we define the \emph{minimum expected Sliced-Wasserstein estimator} (MESWE) \emph{of order $p$} as the estimator obtained by plugging $\swassersteinD[p]$ in place of $\mathbf{D}$ in \eqref{eq:mede}. In the rest of the paper, MSWE and MESWE will be denoted by $\hat{\theta}_{n}$ and $\hat{\theta}_{n,m}$ respectively. 

We present the asymptotic properties that we derived for MSWE and MESWE, namely their existence and consistency. We study their measurability in \Cref{subsec:measurability} of the supplementary document. We also formulate a CLT that characterizes the asymptotic distribution of MSWE and establishes a convergence rate for any dimension. We provide all the proofs in \Cref{sec:postponed-proofs} of the supplementary document. Note that, since the Sliced-Wasserstein distance is an average of one-dimensional Wasserstein distances, some proofs are, inevitably, similar to the proofs done in \cite{Bernton2019}. However, the adaptation of these techniques to the SW case is made possible by the identification of novel properties regarding the topology induced by the SW distance, to the best of our knowledge, which we establish for the first time in this study. %

\subsection{Topology induced by the Sliced-Wasserstein distance}
\label{sec:weak_conv}

We begin this section by a useful result which we believe is interesting on its own and implies that the topology induced by $\swassersteinD[p]$ on $\mcp_p(\rset^d)$ is finer than the weak topology induced by the L\'{e}vy-Prokhorov metric $\dist_{\mathcal{P}}$.
  \begin{theorem} \label{thm:SWp_metrizes_Pp}
    Let $p \in [1, \plusinfty)$. The convergence in $\swassersteinD[p]$ implies the weak convergence in $\calP(\Rd)$. In other words, if $\sequencek{\mu_k}$ is a sequence of measures in $\calP_p(\Rd)$ satisfying %
$     \lim_{k \rightarrow \plusinfty} \swassersteinD[p](\mu_k, \mu) = 0$, with $\mu \in \mcp_p(\rset^d)$, then $\sequencek{\mu_k} \wc \mu$.
 
\end{theorem}
The property that convergence in $\swassersteinD[p]$ implies weak convergence has already been proven in \cite{Bonnotte2013} only for \emph{compact} domains. While the implication of weak convergence is one of the most crucial requirements that a distance metric should satisfy, to the best of our knowledge, this implication has not been proved for general domains before. In \cite{Bonnotte2013}, the main proof technique was based on showing that $\swassersteinD[p]$ is equivalent to $\wassersteinD[p]$ in compact domains, whereas we follow a different path and use the L\'{e}vy characterization.

\subsection{Existence and consistency of MSWE and MESWE}

In our next set of results, we will show that both MSWE and MESWE are consistent, in the sense that, when the number of observations $n$ increases, the estimators will converge to a parameter $\theta_\star$ that minimizes the ideal problem $\theta \mapsto \swassersteinD[p](\mu_\star,\mu_\theta)$. Before we make this argument more precise, let us first present the assumptions that will imply our results. 
\begin{assumption} \label{assumption:continuousmap}
    The map $\theta \mapsto \mu_\theta$ is continuous from $(\Theta,\rho_{\Theta})$ to $(\mcp(\msy),\dist_{\mcp})$, \ie~ for any sequence $(\theta_n)_{n \in \nset}$ in $\Theta$, satisfying  $\lim_{n \to \plusinfty} \rho_\Theta(\theta_n, \theta) = 0$, we have  $\sequencen{\mu_{\theta_n}} \wc \mu_\theta$.
\end{assumption}
\begin{assumption} \label{assumption:datagen}
    The data-generating process is such that $\lim_{n \rightarrow \plusinfty} \swassersteinD[p](\hat{\mu}_n, \mu_\star) = 0$, $\mathbb{P}$-almost surely.
\end{assumption}
\begin{assumption} \label{assumption:boundedset}
There exists $\epsilon > 0$, such that setting  $\epsilon_\star = \inf_{\theta \in \Theta} \swassersteinD[p](\mu_\star, \mu_\theta)$, the set $\Theta^\star_\epsilon = \{ \theta \in \Theta : \swassersteinD[p](\mu_\star, \mu_\theta) \leq \epsilon_\star + \epsilon \}$ is bounded.
\end{assumption}
These assumptions are mostly related to the identifiability of the statistical model and the regularity of the data generating process. They are arguably mild assumptions, analogous to those that have already been considered in the literature \cite{Bernton2019}. Note that, without \Cref{thm:SWp_metrizes_Pp}, the formulation and use of \Cref{assumption:datagen} in our proofs in the supplementary document would not be possible. In the next result, we establish the consistency of MSWE.
\begin{theorem}[Existence and consistency of MSWE] Assume \Cref{assumption:continuousmap}, \Cref{assumption:datagen} and \Cref{assumption:boundedset}. There exists $\mse  \in \mcf$ with $\mathbb{P}(\mse) = 1$ such that, for all $\omega \in \mse$, 
\begin{align}
  \lim_{n \rightarrow \plusinfty} \inf_{\theta \in \Theta} \swassersteinD[p](\hat{\mu}_n(\omega), \mu_\theta) &= \inf_{\theta \in \Theta} \swassersteinD[p](\mu_\star, \mu_\theta), \; \text{ and } \label{eqn:mswe_consist_a} \\
\limsup_{n \rightarrow \plusinfty} \argmin_{\theta \in \Theta} \swassersteinD[p](\hat{\mu}_n(\omega), \mu_\theta) &\subset \argmin_{\theta \in \Theta} \swassersteinD[p](\mu_\star, \mu_\theta) \eqsp, \label{eqn:mswe_consist_b}
\end{align}
where $\hmu_n$ is defined by \eqref{eq:def_empirical_measure_Y}.
Besides, for all $\omega \in \mse$, there exists $n(\omega)$ such that, for all $n \geq n(\omega)$, the set $\argmin_{\theta \in \Theta} \swassersteinD[p](\hat{\mu}_n(\omega), \mu_\theta)$ is non-empty. \label{thm:existence_consistency_mswe}
\end{theorem}
Our proof technique is similar to the one given in \cite{Bernton2019}. This result shows that, when the number of observations goes to infinity, the estimate $\hat{\theta}_n$ will converge to a global minimizer of the problem $\min_{\theta \in \Theta} \swassersteinD[p](\mu_\star, \mu_\theta)$.

In our next result, we prove a similar property for MESWEs as $\min(m,n)$ goes to infinity. In order to increase clarity, and without loss of generality, in this setting, we consider $m$ as a function of $n$ such that $\lim_{n \rightarrow \plusinfty} m(n) = \plusinfty$. 
Now, we derive an analogous version of Theorem~\ref{thm:existence_consistency_mswe} for MESWE. For this result, we need to introduce another continuity assumption. 
\begin{assumption} \label{assumption:sw32}
    If $\lim_{n \rightarrow \plusinfty} \rho_\Theta(\theta_n, \theta) = 0$, then $\lim_{n \rightarrow \plusinfty} \expeLigne{ \swassersteinD[p](\mu_{\theta_n}, \hat{\mu}_{\theta_n, n}) | Y_{1:n}  } = 0$.
\end{assumption}
The next theorem establishes the consistency of MESWE.
\begin{theorem}[Existence and consistency of MESWE] 
Assume \cref{assumption:continuousmap}, \cref{assumption:datagen}, \cref{assumption:boundedset} and \cref{assumption:sw32}. Let $(m(n))_{n \in \nsets}$ be an increasing sequence satisfying $\lim_{n \to \plusinfty} m(n) = \plusinfty$. There exists a set $\mse \subset \Omega$ with $\mathbb{P}(\mse) = 1$ such that, for all $w \in \mse$, 
\begin{align}
  \lim_{n \rightarrow \plusinfty} \inf_{\theta \in \Theta} \expe{ \swassersteinD[p](\hat{\mu}_n, \hat{\mu}_{\theta, m(n)}) \middle| Y_{1:n}} &= \inf_{\theta \in \Theta} \swassersteinD[p](\mu_\star, \mu_\theta), \; \text{ and }  \label{eqn:meswe_consist_a} \\
\limsup_{n \rightarrow \plusinfty} \argmin_{\theta \in \Theta}\ \expe{ \swassersteinD[p](\hat{\mu}_n, \hat{\mu}_{\theta, m(n)}) \middle| Y_{1:n}} &\subset \argmin_{\theta \in \Theta}\ \swassersteinD[p](\mu_\star, \mu_\theta) \eqsp, \label{eqn:meswe_consist_b} 
\end{align}
where $\hat{\mu}_n$ and $ \hat{\mu}_{\theta, m(n)}$ are defined by \eqref{eq:def_empirical_measure_Y} and \eqref{eq:def_empirical_measure_Z} respectively.
Besides, for all $\omega \in \mse$, there exists $n(\omega)$ such that, for all $n \geq n(\omega)$, the set $\argmin_{\theta \in \Theta}\ \expeLigne{ \swassersteinD[p](\hat{\mu}_n, \hat{\mu}_{\theta, m(n)}) | Y_{1:n}}$ is non-empty. \label{thm:existence_consistency_meswe}
\end{theorem}
Similar to Theorem~\ref{thm:existence_consistency_mswe}, this theorem shows that, when the number of observations goes to infinity, the estimator obtained with the expected distance will converge to a global minimizer.

\subsection{Convergence of MESWE to MSWE}

In practical applications, we can only use a finite number of generated samples $Z_{1:m}$. In this subsection, we analyze the case where the observations $Y_{1:n}$ are kept fixed while the number of generated samples increases, \ie~$m \rightarrow \plusinfty$ and we show in this scenario that MESWE converges to MSWE, assuming the latter exists. %

Before deriving this result, we formulate a technical assumption below.

\begin{assumption} \label{assumption:boundedset_eps_n}
    For some $\epsilon > 0$ and $\epsilon_n = \inf_{\theta \in \Theta} \swassersteinD[p](\hat{\mu}_n, \mu_\theta)$, the set $\Theta_{\epsilon, n} = \{ \theta \in \Theta : \swassersteinD[p](\hat{\mu}_n, \mu_\theta) \leq \epsilon_n + \epsilon \}$ is bounded almost surely.
\end{assumption}

\begin{theorem}[MESWE converges to MSWE as $m \rightarrow \plusinfty$] \label{thm:cvg_meswe_to_mswe} Assume \cref{assumption:continuousmap}, \cref{assumption:sw32} and \cref{assumption:boundedset_eps_n}. Then,
  \begin{align}
    \lim_{m \rightarrow \plusinfty} \inf_{\theta \in \Theta} \expe{ \swassersteinD[p](\hat{\mu}_n, \hat{\mu}_{\theta, m})  \middle| Y_{1:n}} &= \inf_{\theta \in \Theta} \swassersteinD[p](\hat{\mu}_n, \mu_\theta) \label{eqn:meswe_to_mswe_a} \\
    \limsup_{m \rightarrow \plusinfty} \argmin_{\theta \in \Theta} \expe{ \swassersteinD[p](\hat{\mu}_n, \hat{\mu}_{\theta, m}) \middle| Y_{1:n}} &\subset \argmin_{\theta \in \Theta} \swassersteinD[p](\hat{\mu}_n, \mu_\theta) \label{eqn:meswe_to_mswe_b}
  \end{align} 
  Besides, there exists $m^*$ such that, for any $ m \geq m^*$, the set $\argmin_{\theta \in \Theta} \expe{ \swassersteinD[p](\hat{\mu}_n, \hat{\mu}_{\theta, m}) | Y_{1:n}}$ is non-empty.
\end{theorem}
This result shows that MESWE would be indeed promising in practice, as one get can more accurate estimations by increasing $m$.

\subsection{Rate of convergence and the asymptotic distribution}

In our last set of theoretical results, we investigate the asymptotic distribution of MSWE and we establish a rate of convergence. 
We now suppose that we are in the well-specified setting, \ie~there exists $\theta_\star$ in the interior of $\Theta$ such that $\mu_{\theta_\star} = \mu_\star$, and we consider the following two assumptions. 
For any $u \in \sphereD$ and $t \in \mathbb{R}$, we define
$F_\theta(u,t) = \int_\msy \1_{\ocint{-\infty,t}} (\ps{u}{y}) \rmd \mu_\theta(y)  $. %
Note that for any $u \in \sphereD$, $F_\theta(u,\cdot)$ is the cumulative distribution function (CDF) associated to the measure $\usss \mu_\theta$.
\begin{assumption} \label{assumption:well_separation}
  For all $\epsilon > 0$, there exists $\delta > 0$ such that 
    $\inf_{\theta \in \Theta :\ \rho_{\Theta}(\theta , \theta_\star) \geq \epsilon} \swassersteinD[1](\mu_{\theta_\star}, \mu_\theta) > \delta \eqsp.$
\end{assumption}
Let $\mcl^1(\sphereD \times \rset)$ denote the class of functions that are absolutely integrable on the domain $\sphereD \times \rset$, with respect to the measure $d\unifS \otimes \Leb $, where $\Leb$ denotes the Lebesgue measure.  
\begin{assumption} \label{assumption:form_derivative}
Assume that there exists a measurable function $D_{\star} = (D_{\star,1}, \dots, D_{\star,d_\theta}) : \sphereD \times \rset \mapsto \mathbb{R}^{d_\theta}$ such that for each $i =1, \dots , d_\theta$, $ D_{\star,i} \in \mcl^1(\sphereD \times \rset)$ and
  \begin{align*}
    \int_{\sphereD} \int_\rset \left| F_{\theta}(u,t) - F_{\theta_\star}(u,t) - \langle \theta - \theta_\star, D_{\star}(u,t) \rangle \right| \rmd t \rmd \unifS(u) = \bfepsilon( \rho_\Theta( \theta , \theta_\star) ) \eqsp,
  \end{align*}
  where $\bfepsilon : \rset_+ \to \rset_+$ satisfies $\lim_{t \to 0} \bfepsilon(t) = 0$. Besides, $\{D_{\star,i}\}_{i=1}^{d_\theta}$ are linearly independent in $\mcl^1(\sphereD \times \rset)$.
\end{assumption}

For any $u \in \sphereD$,  and $t \in \mathbb{R}$, define: 
$\hat{F}_{n}(u,t) = n^{-1} \card\{ i \in \{ 1, \dots, n\} : \ps{u}{Y_i} \leq t \}$,
where $\card$ denotes the cardinality of a set, and for any $u \in \sphereD$, $\hat{F}_n(u,\cdot)$ is the CDF associated to the measure $\usss \hat{\mu}_n$.

\begin{assumption} \label{assumption:weak_convergence_without_norm}
There exists a random element $G_{\star} : \sphereD \times \rset \mapsto \rset$ such that the stochastic process $\sqrt{n} ( \hat{F}_n - F_{\theta_\star} )$ converges weakly in $\mcl_1(\sphereD \times \rset)$ to $G_\star$\footnote{Under mild assumptions on the tails of $\usss \mu_\star$ for any $u \in \sphereD$, we believe that one can prove that \cref{assumption:weak_convergence_without_norm} holds in general by extending \cite[Proposition 3.5]{dede2009empirical} and \cite[Theorem 2.1a]{delbarrio1999}.}.
\end{assumption}

\begin{theorem} \label{thm:asymptotic_1}
  Assume \cref{assumption:continuousmap}, \cref{assumption:datagen}, \cref{assumption:boundedset}, \cref{assumption:well_separation}, \cref{assumption:form_derivative} and \cref{assumption:weak_convergence_without_norm}. Then, the asymptotic distribution of the goodness-of-fit statistic is given by,
  \begin{equation*} 
    \sqrt{n} \inf_{\theta \in \Theta} \swassersteinD[1](\hat{\mu}_n, \mu_\theta) \wc \inf_{\theta \in \Theta} \int_{\sphereD} \int_\rset \left| G_{\star}(u,t) - \langle \theta, D_{\star}(u,t) \rangle \right| \rmd t \rmd \unifS(u), \quad \text{ as }  n \rightarrow \plusinfty \eqsp,
  \end{equation*}
  where $\hmu_n$ is defined by \eqref{eq:def_empirical_measure_Y}.
\end{theorem}

\begin{theorem} \label{thm:asymptotic_2}
  Assume \cref{assumption:continuousmap}, \cref{assumption:datagen}, \cref{assumption:boundedset}, \cref{assumption:well_separation}, \cref{assumption:form_derivative} and \cref{assumption:weak_convergence_without_norm}. Suppose also that the random map $\theta \mapsto \int_{\sphereD} \int_\rset \left| G_{\star}(u,t) - \langle \theta, D_{\star}(u,t) \rangle \right| \rmd t \rmd \unifS(u)$ has a unique infimum almost surely. 
  Then, MSWE with $p = 1$ satisfies,
  \begin{equation*}
    \sqrt{n} ( \hat{\theta}_n - \theta_\star ) \wc \argmin_{\theta \in \Theta} \int_{\sphereD} \int_\rset \left| G_{\star}(u,t) - \langle \theta, D_{\star}(u,t) \rangle \right| \rmd t \rmd \unifS(u), \quad \text{ as }  n \rightarrow \plusinfty \eqsp,
  \end{equation*}
  where $\htheta_n$ is defined by \eqref{eq:mde} with $\swassersteinD[1]$ in place of $    \mathbf{D}$.
\end{theorem}

These results show that the estimator and the associated goodness-of-fit statistics will converge to a random variable in distribution, where the rate of convergence is $\sqrt{n}$. Note that $G_{\star}$ is defined as a random element (see \Cref{assumption:weak_convergence_without_norm}), therefore we can not claim that the convergence in distribution derived in \Cref{thm:asymptotic_1} and \ref{thm:asymptotic_2} implies the convergence in probability. 

This CLT is also inspired by \cite{Bernton2019}, where they identified the asymptotic distribution associated to the minimum Wasserstein estimator. However, since $\wassersteinD[p]$ admits an analytical form only when $d=1$, their result is restricted to the scalar case, and in their conclusion, \cite{Bernton2019} conjecture that the rate of the minimum Wasserstein estimators would depend negatively on the dimension of the observation space. On the contrary, since $\swassersteinD[p]$ is defined in terms of one-dimensional $\wassersteinD[p]$ distances, we circumvent the curse of dimensionality and our result holds for any finite dimension. While the perceived computational burden has created a pessimism in the machine learning community about the use of Wasserstein-based methods in large dimensional settings, which motivated the rise of regularized optimal transport \cite{peyre2019computational}, we believe that our findings provide an interesting counter-example to this conception.

\newpage 
\section{Experiments}

We conduct experiments on synthetic and real data to empirically confirm our theorems. We explain in \Cref{appendix:computational} of the supplementary document the optimization methods used to find the estimators. Specifically, we can use stochastic iterative optimization algorithm (e.g., stochastic gradient descent). Note that, since we calculate (expected) SW with Monte Carlo approximations over a finite set of projections (and a finite number of \emph{`generated datasets'}), MSWE and MESWE fall into the category of doubly stochastic algorithms. Our experiments on synthetic data actually show that using only one random projection and one randomly generated dataset at each iteration of the optimization process is enough to illustrate our theorems. We provide the code to reproduce the experiments.\footnote{See \url{https://github.com/kimiandj/min_swe}.}

\begin{wrapfigure}{r}{0.5\textwidth}
  \vspace{-12pt} 
  \centering
  \includegraphics[width=0.41\textwidth]{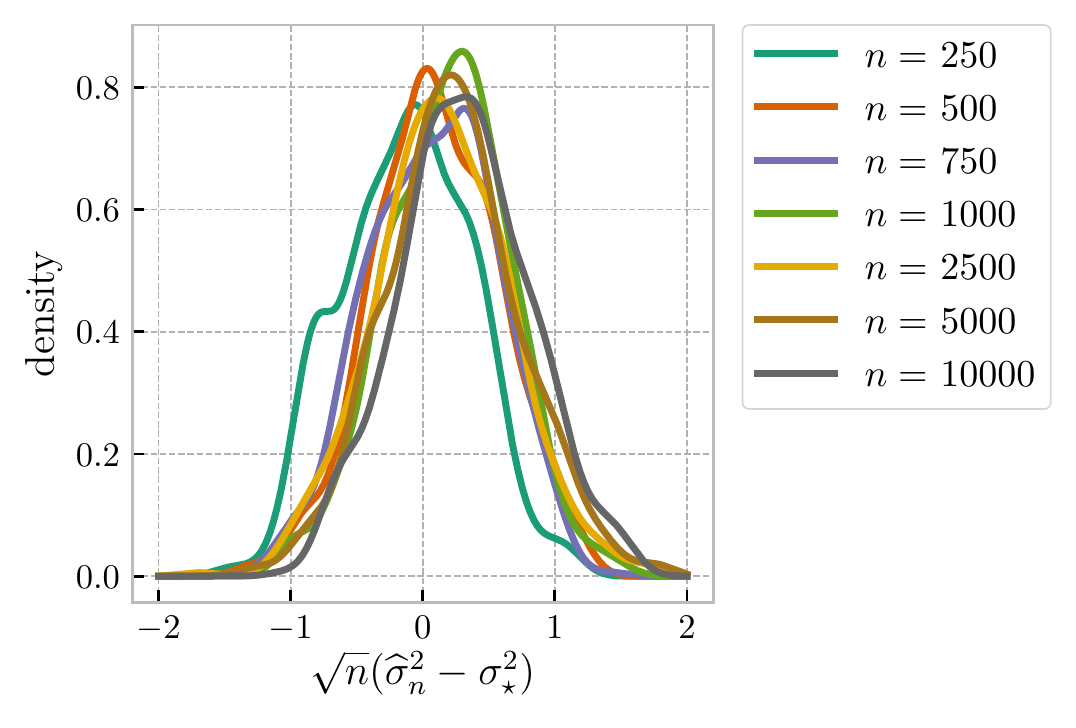}
  \caption{Probability density estimates of the MSWE $\hat{\sigma}^2_n$ of order 1, centered and rescaled by $\sqrt{n}$, on the 10-dimensional Gaussian model for different values of $n$.}
  \label{fig:results_gaussian_asymptotic}
  \vspace{-5pt}
\end{wrapfigure} 
\begin{figure}[t]
  \centering
    \subfloat[MSWE vs. $n$]{
      \includegraphics[width=0.32\textwidth]{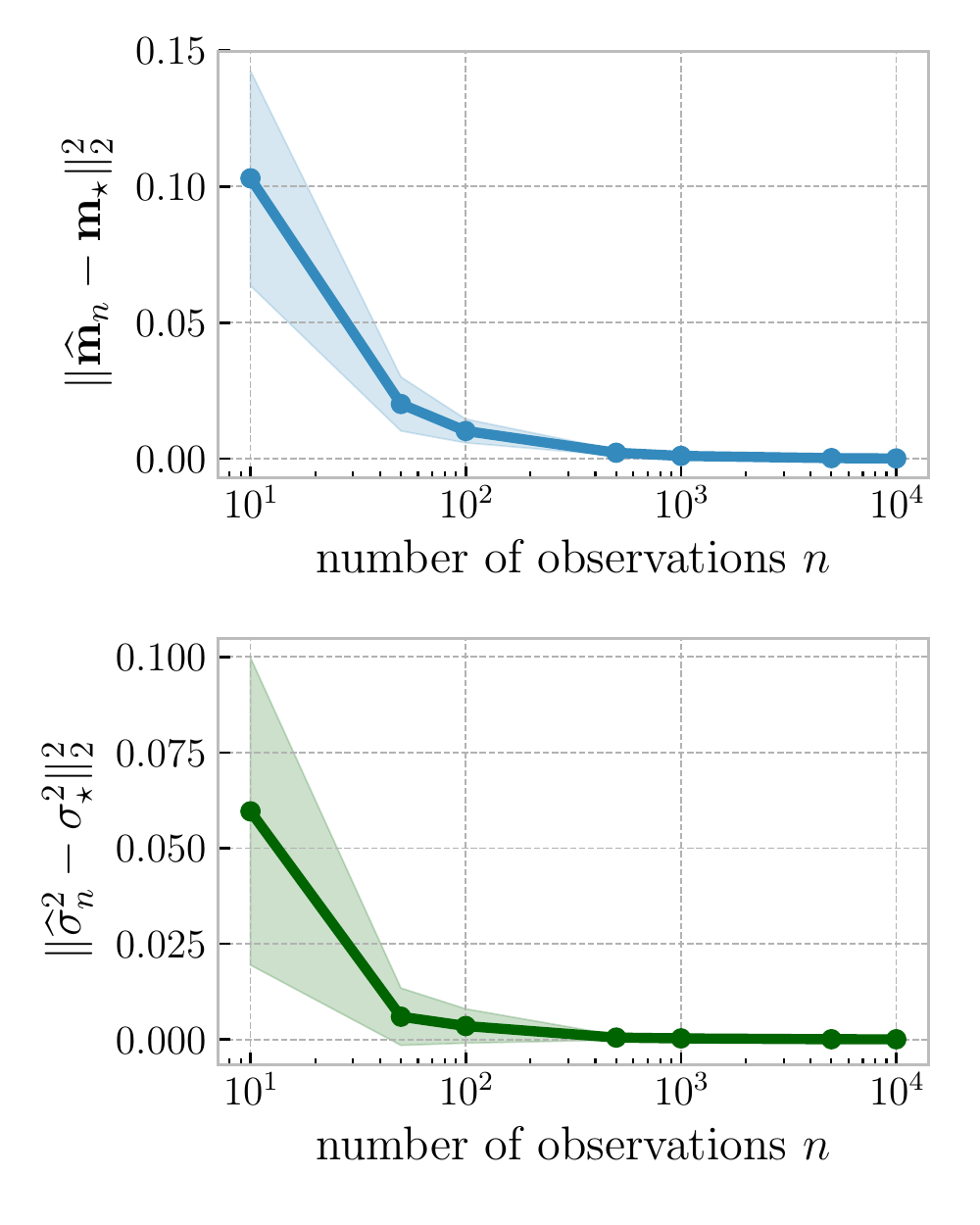}
    \label{fig:results_gaussian_exp_a}
    }
    \subfloat[MESWE vs. $n = m$]{
      \includegraphics[width=0.32\textwidth]{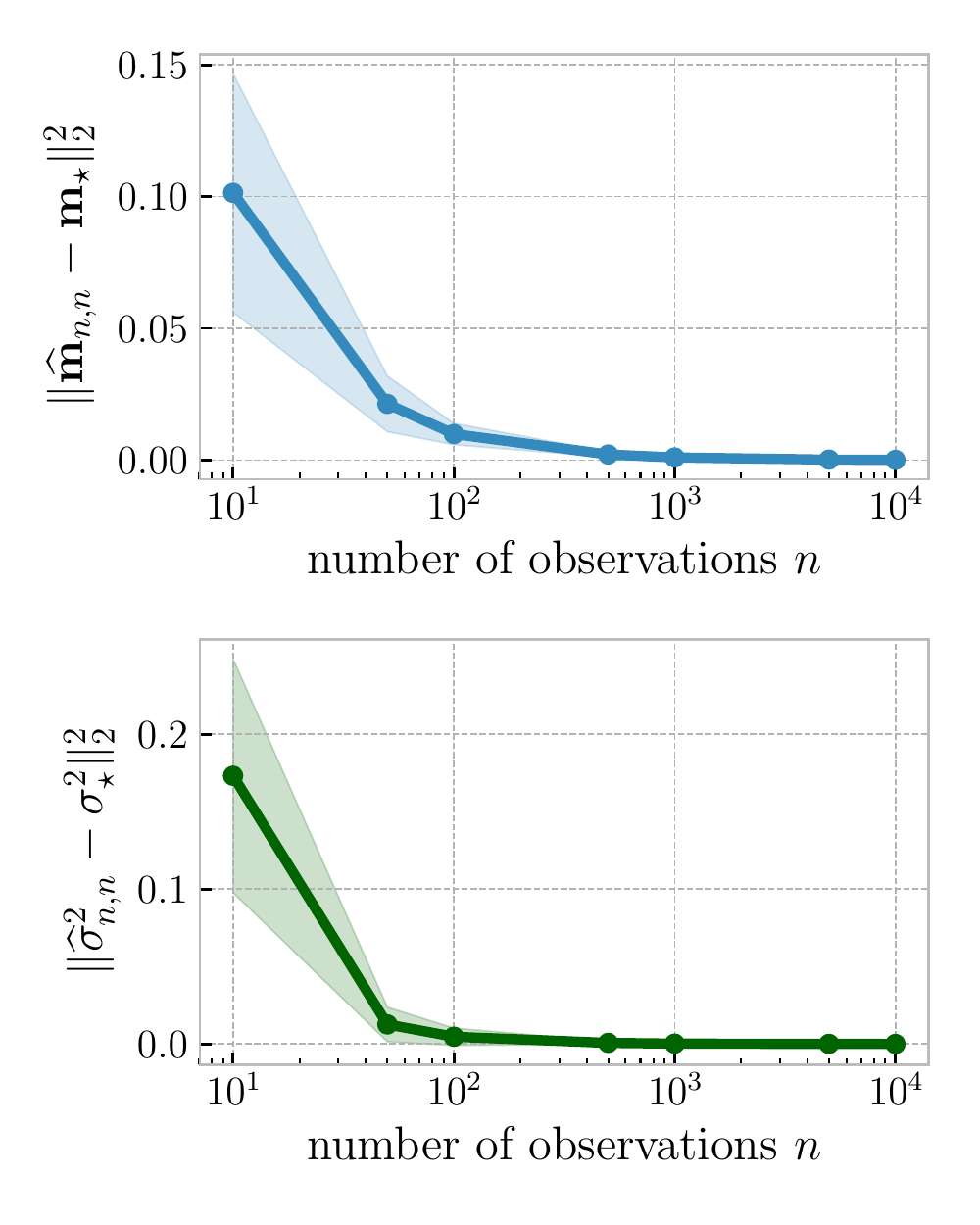}
    \label{fig:results_gaussian_exp_b}
    }
    \subfloat[MESWE with $n=2000$ vs. $m$]{
      \includegraphics[width=0.32\textwidth]{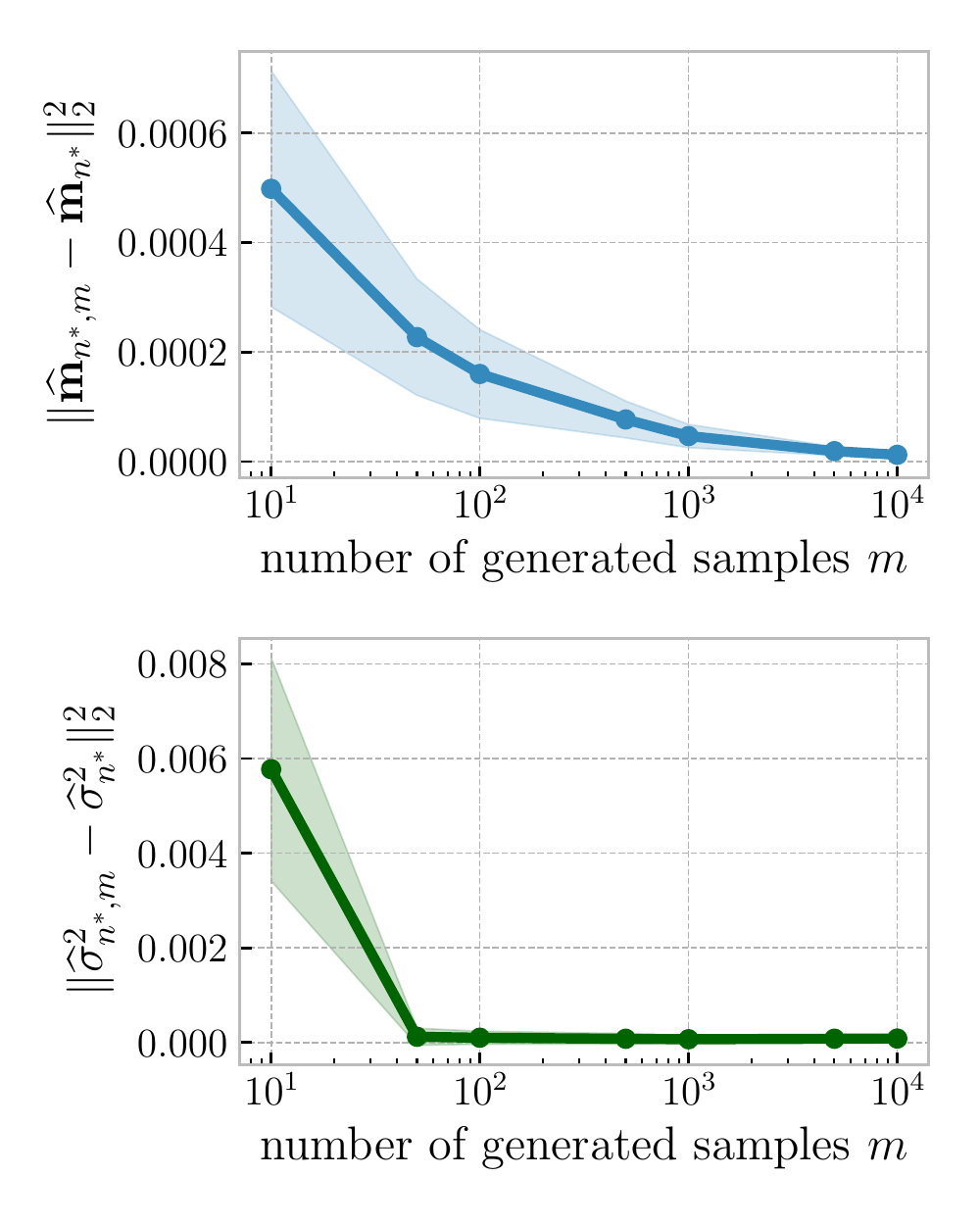}
    \label{fig:results_gaussian_exp_c}
    } 
    \caption{Min. SW estimation on Gaussians in $\rset^{10}$. \Cref{fig:results_gaussian_exp_a} and \Cref{fig:results_gaussian_exp_b} show the mean squared error between $(\bfm_\star, \sigma^2_\star) = (\bfnot, 1)$ and MSWE $(\hat{\bfm}_n, \hat{\sigma}^2_n)$ (resp. MESWE $(\hat{\bfm}_{n,n}, \hat{\sigma}^2_{n,n})$) for $n$ from 10 to 10\,000, illustrating Theorems~\ref{thm:existence_consistency_mswe} and \ref{thm:existence_consistency_meswe}. \Cref{fig:results_gaussian_exp_c} shows the error between $(\hat{\bfm}_n, \hat{\sigma}^2_n)$ and $(\hat{\bfm}_{n,m}, \hat{\sigma}^2_{n,m})$ for 2000 observations and $m$ from 10 to 10\,000, to illustrate \Cref{thm:cvg_meswe_to_mswe}. Results are averaged over 100 runs, the shaded areas represent the standard deviation.}\label{fig:results_gaussian_exp}
    \vspace{-5pt}
\end{figure}

\textbf{Multivariate Gaussian distributions:} We consider the task of estimating the parameters of a 10-dimensional Gaussian distribution using our SW estimators: we are interested in the model $\calM = \left\{ \calN(\bfm, \sigma^2\bfI)\ :\ \bfm \in \rset^{10},\ \sigma^2 > 0 \right\}$ and we draw i.i.d. observations with $(\bfm_\star, \sigma^2_\star) = (\bfnot, 1)$. The advantage of this simple setting is that the density of the generated data has a closed-form expression, which makes MSWE tractable. We empirically verify our central limit theorem: for different values of $n$, we compute 500 times MSWE of order 1 using one random projection, then we estimate the density of $\hat{\sigma}^2_n$ with a kernel density estimator. \Cref{fig:results_gaussian_asymptotic} shows the distributions centered and rescaled by $\sqrt{n}$ for each $n$, and confirms the convergence rate that we derived (\Cref{thm:asymptotic_2}). To illustrate the consistency property in \Cref{thm:existence_consistency_mswe}, we approximate MSWE of order 2 for different numbers of observed data $n$ using one random projection and we report for each $n$ the mean squared error between the estimate mean and variance and the data-generating parameters $(\bfm_\star, \sigma^2_\star)$. We proceed the same way to study the consistency of MESWE (\Cref{thm:existence_consistency_meswe}), which we approximate using one random projections and one generated dataset $z_{1:m}$ of size $m = n$ for different values of $n$. We also verify the convergence of MESWE to MSWE (\Cref{thm:cvg_meswe_to_mswe}): we compute these estimators on a fixed set of $n = 2000$ observations for different $m$, and we measure the error between them for each $m$. Results are shown in \Cref{fig:results_gaussian_exp}. We see that our estimators indeed converge to $(\bfm_\star, \sigma^2_\star)$ as the number of observations increases (Figures~\ref{fig:results_gaussian_exp_a}, \ref{fig:results_gaussian_exp_b}), and on a fixed observed dataset, MESWE converges to MSWE as we generate more samples (\Cref{fig:results_gaussian_exp_c}).

\textbf{Multivariate elliptically contoured stable distributions:} We focus on parameter inference for a subclass of multivariate stable distributions, called elliptically contoured stable distributions and denoted by $\ellstable$ \cite{Nolan2013}. Stable distributions refer to a family of heavy-tailed probability distributions that generalize Gaussian laws and appear as the limit distributions in the generalized central limit theorem \cite{samorodnitsky1994stable}. These distributions have many attractive theoretical properties and have been proven useful in modeling financial \cite{mandelbrot2013fractals} data or audio signals \cite{csimcsekli2015alpha,leglaive2017alpha}. While special univariate cases include Gaussian, L\'{e}vy and Cauchy distributions, the density of stable distributions has no general analytic form, which restricts their practical application, especially for the multivariate case. 

If $Y \in \Rd \sim \ellstable(\bfSigma, \bfm)$, then its joint characteristic function is defined for any $\bft \in \Rd$ as $ \Esp [ \exp (i\bft^T Y) ] = \exp \left( - (\bft^T \bfSigma \bft)^{\alpha / 2} + i \bft^T \bfm \right)$, where $\bfSigma$ is a positive definite matrix (akin to a correlation matrix), $\bfm \in \Rd$ is a location vector (equal to the mean if it exists) and $\alpha \in (0, 2)$ controls the thickness of the tail. Even though their densities cannot be evaluated easily, it is straightforward to sample from $\ellstable$ \cite{Nolan2013}, therefore it is particularly relevant here to apply MESWE instead of MLE. 

To demonstrate the computational advantage of MESWE over the minimum expected Wasserstein estimator \cite[MEWE]{Bernton2019}, we consider observations in $\Rd$ i.i.d. from $\ellstable(\bfI, \bfm_\star)$ where each component of $\bfm_\star$ is 2 and $\alpha = 1.8$, and $\calM = \left\{ \ellstable(\bfI, \bfm)\ :\ \bfm \in \Rd \right\}$. The Wasserstein distance on multivariate data is either computed exactly by solving the linear program in \eqref{eq:def_wasser}, or approximated by solving a regularized version of this problem with Sinkhorn's algorithm \cite{cuturi2013sinkhorn}. The MESWE is approximated using 10 random projections and 10 sets of generated samples. Then, following the approach in \cite{Bernton2019}, we use the gradient-free optimization method Nelder-Mead to minimize the Wasserstein and SW distances. We report on \Cref{fig:results_comparison} the mean squared error between each estimate and $\bfm_\star$, as well as their average computational time for different values of dimension $d$. We see that MESWE provides the same quality of estimation as its Wasserstein-based counterparts while considerably reducing the computational time, especially in higher dimensions. We focus on this model in $\rset^{10}$ and we illustrate the consistency of the MESWE $\hat{\bfm}_{n,m}$, approximated with one random projection and one generated dataset, the same way as for the Gaussian model: see \Cref{fig:results_alphastable_exp_a}. To confirm the convergence of $\hat{\bfm}_{n,m}$ to the MSWE $\hat{\bfm}_n$, we fix $n=100$ observations and we compute the mean squared error between the two approximate estimators (using one random projection and one generated dataset) for different values of $m$ (\Cref{fig:results_alphastable_exp_b}). Note that the MSWE is approximated with the MESWE obtained for a large enough value of $m$: $\hat{\bfm}_n \approx \hat{\bfm}_{n, 10\,000}$.

\begin{figure}[t]
  \centering
  \subfloat[Comparison of MESWE and MEWE]{
      \includegraphics[width=0.48\textwidth]{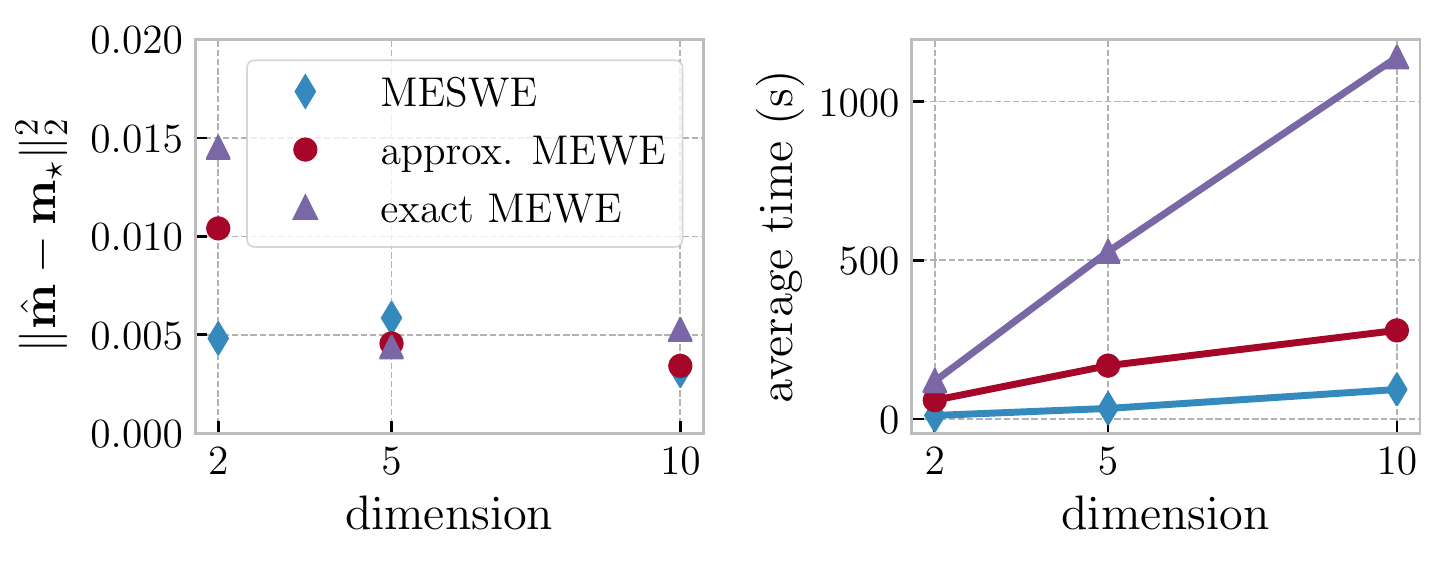}
    \label{fig:results_comparison}
    } %
    \subfloat[MESWE]{
      \includegraphics[width=0.22\textwidth]{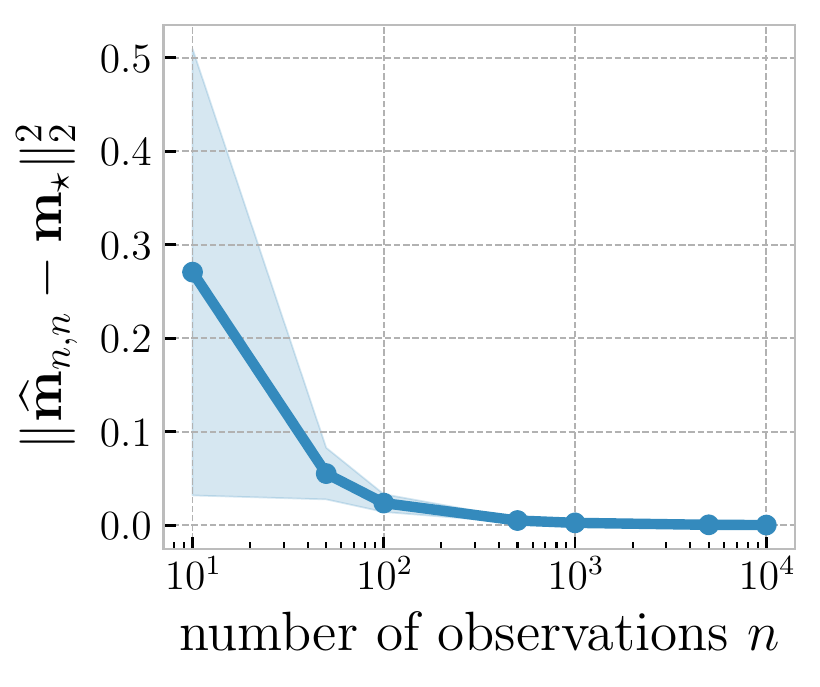}
    \label{fig:results_alphastable_exp_a}
    } %
    \subfloat[MESWE, $n^* = 100$]{
      \includegraphics[width=0.24\textwidth]{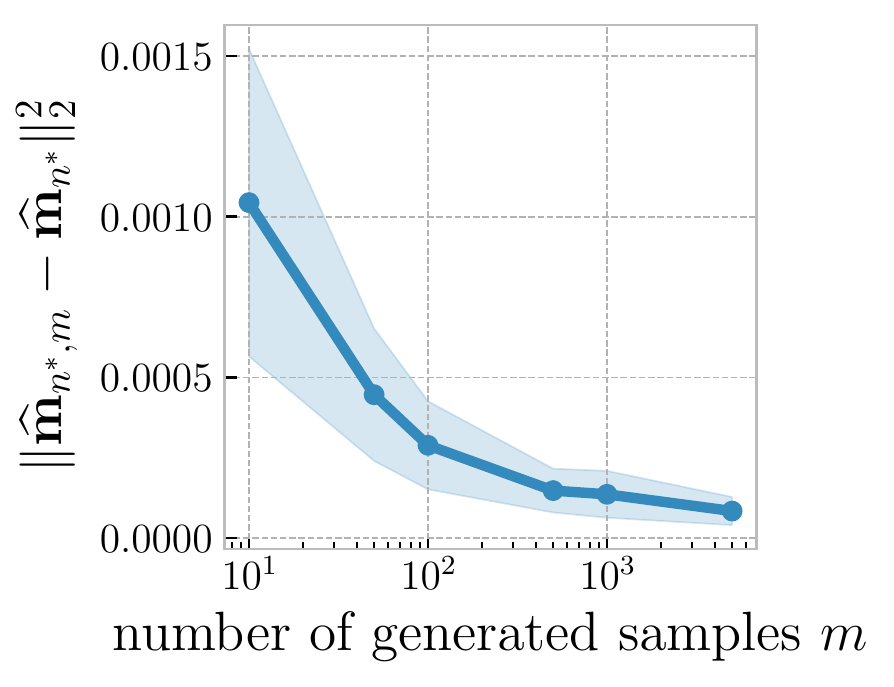}
    \label{fig:results_alphastable_exp_b}
    }
    \caption{Min. SW estimation for the location parameter of multivariate elliptically contoured stable distributions. \Cref{fig:results_comparison} compares the quality of the estimation provided by SW and Wasserstein-based estimators as well as their average computational time, for different values of dimension $d$. \Cref{fig:results_alphastable_exp_a} and \Cref{fig:results_alphastable_exp_b} illustrate, for $d=10$, the consistency of MESWE $\hat{\bfm}_{n,m}$ and its convergence to the MSWE $\hat{\bfm}_n$. Results are averaged over 100 runs, the shaded area represent the standard deviation.}
    \vspace{-5pt}
\end{figure}

\textbf{High-dimensional real data using GANs:} Finally, we run experiments on image generation using the Sliced-Wasserstein Generator (SWG), an alternative GAN formulation based on the minimization of the SW distance \cite{Deshpande2018}. Specifically, the generative modeling approach consists in introducing a random variable $Z$ which takes value in $\msz$ with a fixed distribution, and then transforming $Z$ through a neural network. This defines a parametric function $T_\theta : \msz \rightarrow \msy$ that is able to produce images from a distribution $\mu_\theta$.
\newpage
 \begin{wrapfigure}{r}{0.6\textwidth}
 \vspace{-5pt}
  \centering
  \includegraphics[width=0.4\textwidth]{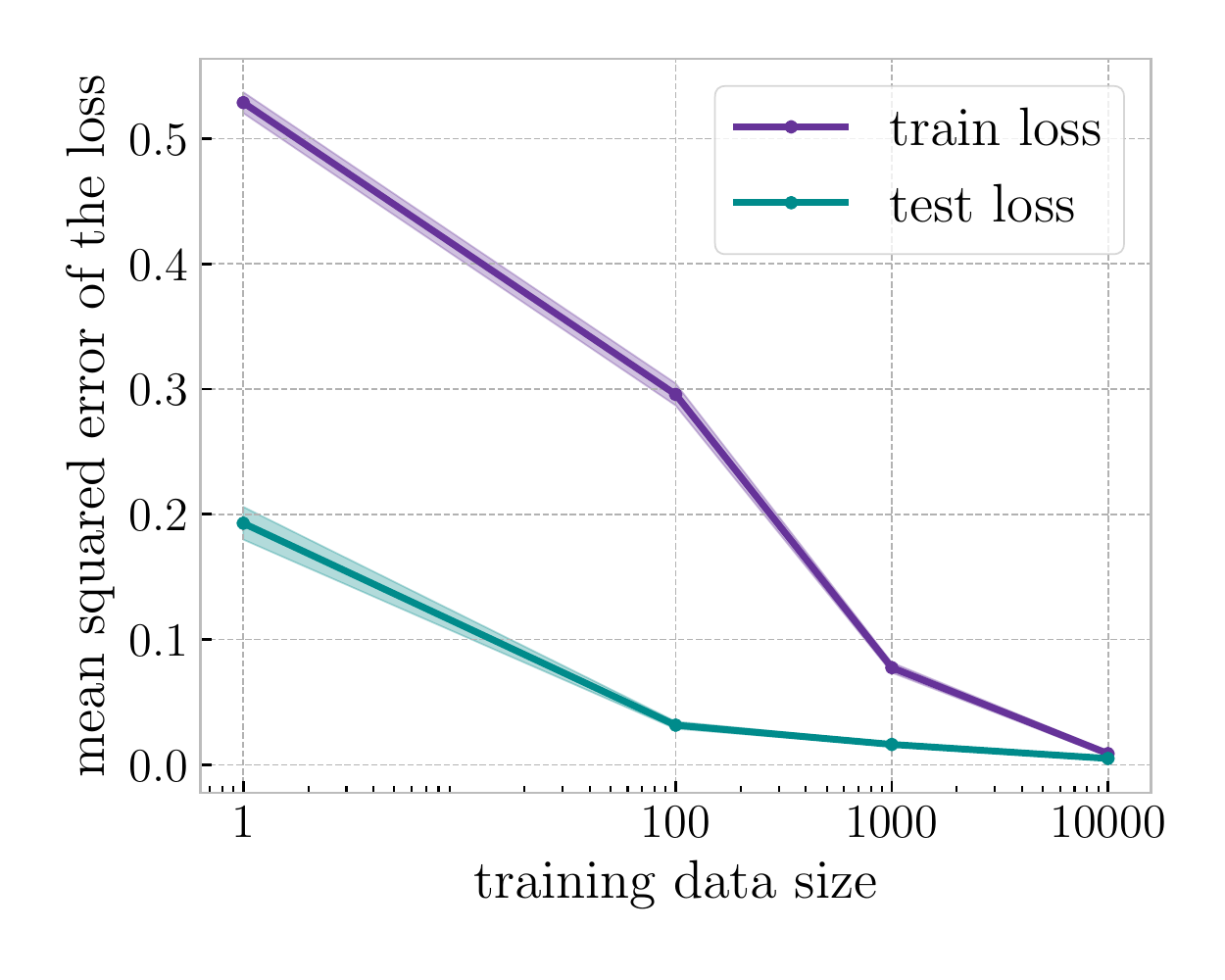}
  \vspace{-10pt}
  \caption{Mean-squared error between the training (test) loss for $(n,m) \in \big\{ (1, 1), (100, 20), (1000, 40), (10\,000, 60)\big\}$ and the training (test) loss for $(n, m) = (60\,000, 200)$ on MNIST using the SW generator. We trained for 20\,000 iterations with the ADAM optimizer \cite{Adam}.}
  \label{fig:results_nn_exp_a}
  \vspace{-8pt}
\end{wrapfigure}
The goal is to optimize the neural network parameters such that the generated images are close to the observed ones. \cite{Deshpande2018} proposes to minimize the SW distance between $\mu_\theta$ and the real data distribution over $\theta$ as the generator objective, and train on MESWE in practice. For our experiments, we design a neural network with the fully-connected configuration given in \cite[Appendix D]{Deshpande2018} and we use the MNIST dataset, made of 60\,000 training images and 10\,000 test images of size $28 \times 28$. Our training objective is MESWE of order 2 approximated with 20 random projections and 20 different generated datasets. We study the consistent behavior of the MESWE by training the neural network on different sizes $n$ of training data and different numbers $m$ of generated samples and by comparing the final training loss and test loss to the ones obtained when learning on the whole training dataset ($n = 60\,000$) and $m = 200$. Results are averaged over 10 runs and shown on \Cref{fig:results_nn_exp_a}, where the shaded areas correspond to the standard deviation over the runs. We observe that our results confirm \Cref{thm:existence_consistency_meswe}. 

We would like to point out that, in all of our experiments, the random projections used in the Monte Carlo average that estimates the integral in \eqref{eq:def_sliced_wasser} were picked uniformly on $\sphereD$ (see \Cref{appendix:computational} in the supplementary document for more details). The sampling on $\sphereD$ directly impacts the quality of the resulting approximation of SW, and might induce variance in practice when learning generative models. On the theoretical side, studying the asymptotic properties of SW-based estimators obtained with a finite number of projections is an interesting question (e.g., their behavior might depend on the sampling method or the number of projections used). We leave this study for future research.

\section{Conclusion}

The Sliced-Wasserstein distance has been an attractive metric choice for learning in generative models, where the densities cannot be computed directly. In this study, we investigated the asymptotic properties of estimators that are obtained by minimizing SW and the expected SW. We showed that (i) convergence in SW implies weak convergence of probability measures in general Wasserstein spaces, (ii) the estimators are consistent, (iii) the estimators converge to a random variable in distribution with a rate of $\sqrt{n}$. We validated our mathematical results on both synthetic data and neural networks. We believe that our techniques can be further extended to the extensions of SW such as \cite{deshpande2019max,paty2019subspace,kolouri2019generalized}.

\subsection*{Acknowledgements}

The authors are grateful to Pierre Jacob for his valuable comments on an earlier version of this manuscript. This work is partly supported by the French National Research Agency (ANR) as a part of the FBIMATRIX project (ANR-16-CE23-0014) and by the industrial chair Machine Learning for Big Data from Télécom ParisTech. Alain Durmus acknowledges support from Polish National Science Center grant: NCN UMO-2018/31/B/ST1/00253.

\bibliography{min_swe_arxiv.bib}
\bibliographystyle{unsrt}

\newpage
\appendix  

\section{Preliminaries}

\subsection{Convergence and lower semi-continuity}

\begin{definition}[Weak convergence] 
  Let $\sequencek{\mu_k}$ be a sequence of probability measures on $\msy$. We say that $\mu_k$ converges weakly to a probability measure $\mu$ on $\msy$, and write $\sequencek{\mu_k} \wc \mu$ (or $\mu_k \wc \mu$), if for any continous and bounded function $f$ : $\msy \rightarrow \rset$, we have
$$ \lim_{k \to \plusinfty}\int f\ \rmd\mu_k = \int f\ \rmd\mu \eqsp.$$
\end{definition} %

\begin{definition}[Epi-convergence]
  Let $\Theta$ be a metric space and $f : \Theta \rightarrow \rset$. Consider a sequence  $\sequencek{f_k}$ of functions from $\Theta$ to $\rset$. We say that the sequence $\sequencek{f_k}$ epi-converges to a function $f : \Theta \rightarrow \rset$, and write $\sequencek{f_k} \ec f$, if for each $\theta \in \Theta$,
  \begin{align*}
    \liminf_{k \rightarrow \infty} f_k(\theta_k) &\geq f(\theta) \; \text{ for every sequence } (\theta_k)_{n \in \nset} \text{ such that } \lim_{k \to \plusinfty} \theta_k =  \theta \eqsp,\\
\text{ and } \quad     \limsup_{k \rightarrow \infty} f_k(\theta_k) &\leq f(\theta) \; \text{ for a sequence } (\theta_k)_{n \in \nset} \text{ such that } \lim_{k \to \plusinfty} \theta_k =  \theta \eqsp.
  \end{align*}
\end{definition}

An equivalent and useful characterization of epi-convergence is given in \cite[Proposition 7.29]{Rockafellar2009}, which we paraphrase in Proposition \ref{prop:729rockafellar} after recalling the definition of lower semi-continuous functions. 

\begin{definition}[Lower semi-continuity]
  Let $\Theta$ be a metric space and $f : \Theta \rightarrow \rset$. We say that $f$ is lower semi-continuous (\lsc) on $\Theta$ if for any $\theta_0 \in \Theta$, 
  \begin{equation*} 
    \liminf_{\theta \rightarrow \theta_0} f(\theta) \geq f(\theta_0)
  \end{equation*}
\end{definition}

\begin{proposition}[Characterization of epi-convergence via minimization, Proposition 7.29 of \cite{Rockafellar2009}] \label{prop:729rockafellar}
    Let $\Theta$ be a metric space and $f : \Theta \rightarrow \rset$ be a \lsc~function.
  The sequence $\sequencek{f_k}$, with $f_k : \Theta \rightarrow \rset$~for any $n \in \nset$, epi-converges to $f$  if and only if
  \begin{itemize}
    \item[(a)] $\liminf_{k \rightarrow \infty} \inf_{\theta \in \msk} f_k(\theta) \geq \inf_{\theta \in \msk} f(\theta)$ \; for every compact set $\msk \subset \Theta$ ; 
    \item[(b)] $\limsup_{k \rightarrow \infty} \inf_{\theta \in \mso} f_k(\theta) \leq \inf_{\theta \in \mso} f(\theta)$ \; for every open set $\mso \subset \Theta$.
  \end{itemize}
\end{proposition}

 \cite[Theorem 7.31]{Rockafellar2009}, paraphrased below, gives asymptotic properties for the infimum and argmin of epiconvergent functions and will be useful to prove the existence and consistency of our estimators.

 \begin{theorem}[Inf and argmin in epiconvergence, Theorem 7.31 of \cite{Rockafellar2009}] \label{thm:731rockafellar}
     Let $\Theta$ be a metric space, $f : \Theta \rightarrow \rset$ be a \lsc~function and $\sequencek{f_k}$ be a sequence with $f_k : \Theta \rightarrow \rset$~for any $n \in \nset$.
  Suppose $\sequencek{f_k} \ec f$ with $- \infty < \inf_{\theta \in \Theta} f(\theta) < \infty$. 
  \begin{itemize}
    \item[(a)] It holds $\lim_{k \to \infty} \inf_{\theta \in \Theta} f_k(\theta) = \inf_{\theta \in \Theta} f(\theta)$ if and only if  for every $\eta > 0$ there exists a compact set $\msk \subset \Theta$ and $N \in \nset$ such for any $k \geq N$,
      \begin{equation*}
        \inf_{\theta \in \msk} f_k(\theta) \leq \inf_{\theta \in \Theta} f_k(\theta) + \eta \eqsp. 
      \end{equation*}
    \item[(b)] In addition, $\limsup_{k\rightarrow \infty} \argmin_{\theta \in \Theta} f_k(\theta) \subset \argmin_{\theta \in \Theta} f(\theta)$.
  \end{itemize}
\end{theorem}

\section{Preliminary results} 

In this section, we gather technical results regarding lower semi-continuity of (expected) Sliced-Wasserstein distances and measurability of MSWE which will be needed in our proofs. 
\subsection{Lower semi-continuity of Sliced-Wasserstein distances}

\begin{lemma}[Lower semi-continuity of $\mathbf{SW}_p$] \label{lem:sw_semicontinuous}
Let $p \in [1,\infty)$. The Sliced-Wasserstein distance of order $p$ is lower semi-continuous on $\calP_p(\msy) \times \calP_p(\msy)$ endowed with the topology of weak convergence, \ie~for any sequences $\sequencek{\mu_k}$ and $\sequencek{\nu_k}$ of $\calP_p(\msy)$ which converge weakly to $\mu \in \mcp_p(\msy)$ and $\nu\in\mcp_p(\msy)$ respectively, we have:
  \begin{equation*}
    \swassersteinD[p](\mu,\nu) \leq \liminf_{k \to \plusinfty}     \swassersteinD[p] \, (\mu_k,\nu_k) \eqsp. 
  \end{equation*}
\end{lemma}

\begin{proof}
  First, by the continuous mapping theorem, if a sequence $\sequencek{\mu_k}$ of elements of $\calP_p(\msy)$ converges weakly to $\mu$, then for any continuous function $f : \msy \to \rset$, $\sequencek{f_{\sharp} \mu_k}$ converges weakly to $f_{\sharp} \mu$. In particular, for any $\us \in \sphere^{d-1}$, $\usss\mu_k \wc \usss\mu$ since $\uss$ is a bounded linear form thus continuous.
    
  Let $p \in [1,\infty)$. We introduce the two sequences $\sequencek{\mu_k}$ and $\sequencek{\nu_k}$ of elements of $\calP_p(\msy)$ such that $\mu_k \wc \mu$ and $\nu_k \wc \nu$. We show that for any $\us \in \sphereD$,
  \begin{equation}
    \label{eq:lsc_proof_0}
    \wassersteinD[p]^p(\usss \mu, \usss \nu) \leq \liminf_{k \to \plusinfty}     \wassersteinD[p]^p(\usss \mu_k , \usss \nu_k) \eqsp.
  \end{equation}

  Indeed, if \eqref{eq:lsc_proof_0} holds, then the proof is completed using the definition of the Sliced-Wasserstein distance  \eqref{eq:def_sliced_wasser} and Fatou's Lemma. Let $\us \in \sphereD$. For any $k \in \nset$, let $\gamma_k\in \mcp(\rset \times \rset)$ be an  optimal transference plan between $\usss \mu_k$ and $\usss \nu_k$ for the Wasserstein distance of order $p$ which exists by \cite[Theorem 4.1]{Villani2008}  \ie~
  \begin{equation*}
    \wassersteinD[p]^p(\usss \mu_k , \usss \nu_k)  = \int_{\rset \times \rset} \abs{a-b} \rmd \gamma_k(a,b) \eqsp.
  \end{equation*}
 Note that by \cite[Lemma 4.4]{Villani2008} and Prokhorov's Theorem, $\sequencek{\gamma_k}$ is sequentially compact in $\mcp(\rset\times \rset)$ for the topology associated with the weak convergence. Now, consider a subsequence $\sequencek{\gamma_{\phi_1(k)}}$ where $\phi_1 : \nset \to \nset$ is increasing such that
  \begin{align}
    \lim_{k \to \plusinfty} \int_{\rset \times \rset} \abs{a-b}^p \rmd \gamma_{\phi_1(k)}(a,b) &= \lim_{k \to \plusinfty}   \wassersteinD[p]^p(\usss \mu_{\phi_1(k)}, \usss \nu_{\phi_1(k)}) \nonumber \\
    &= \liminf_{k \to \plusinfty}  \wassersteinD[p]^p(\usss \mu_k, \usss\nu_k) \eqsp. \label{eq:lsc_proof_1}
  \end{align}
  Since  $\sequencek{\gamma_k}$ is sequentially compact,  $\sequencek{\gamma_{\phi_1(k)}}$ is sequentially compact as well, and therefore there exists an increasing function $\phi_2 : \nset \to \nset$ and a probability distribution $\gamma \in \mcp(\rset \times \rset)$ such that $\sequencek{\gamma_{\phi_2(\phi_1(k))}}$ converges weakly to $\gamma$. Then, we obtain by \eqref{eq:lsc_proof_1}, 
  \begin{equation*}
\int_{\rset \times \rset} \norm{a-b}^p \rmd \gamma(a,b) =  \lim_{k \to \plusinfty} \int_{\rset \times \rset} \norm{a-b}^p \rmd \gamma_{\phi_2(\phi_1(k))}(a,b) =  \liminf_{k \to \plusinfty}  \wassersteinD[p]^p(\usss \mu_k, \usss \nu_k) \eqsp. 
\end{equation*}
If we show that $\gamma \in \Gamma(\usss \mu, \usss \nu)$, it will conclude the proof of \eqref{eq:lsc_proof_0} by definition of the Wasserstein distance \eqref{eq:def_wasser}. But for any continuous and bounded function $f : \rset \to \rset$, since for any $k \in \nset$, $\gamma_k \in \Gamma(\mu_k,\nu_k)$, and $\sequencek{\mu_k},\sequencek{\nu_k}$ converge weakly to $\mu$ and $\nu$ respectively, we have:
\begin{multline*}
  \int_{\rset \times \rset} f(a) \rmd \gamma(a,b) = \lim_{k \to \plusinfty} \int_{\rset \times \rset} f(a) \rmd \gamma_{\phi_2(\phi_1(k))}(a,b)  =
  \lim_{k \to \plusinfty} \int_{\rset} f(a) \rmd \usss \mu_{\phi_2(\phi_1(k))}(a) \\  = \int_{\rset} f(a) \rmd \usss \mu(a) \eqsp,
\end{multline*}
and similarly
\begin{equation*}
  \int_{\rset \times \rset} f(b) \rmd \gamma(a,b) = \int_{\rset} f(b) \rmd \usss \nu(a) \eqsp. 
\end{equation*}

This shows that $\gamma \in \Gamma(\usss \mu, \usss \nu)$ and therefore, \eqref{eq:lsc_proof_0} is true. We conclude by applying Fatou's Lemma. 

\end{proof} 

By a direct application of \Cref{lem:sw_semicontinuous}, we have the following result. 
\begin{corollary} \label{coro:sw_semicontinuous_2}
  Assume \Cref{assumption:continuousmap}. Then, $(\mu, \theta) \mapsto \swassersteinD[p](\mu, \mu_\theta)$ is lower semi-continuous in $\mathcal{P}_p(\msy) \times \Theta$.
\end{corollary}

\begin{lemma}[Lower semi-continuity of $\Esp \mathbf{SW}_p$] \label{lemma:lsc_Esw}
  Let $p \in [1, \infty)$ and $m \in \bbN^*$. Denote for any $\mu \in \mcp_p(\msy)$, $\hat{\mu}_m = (1/m) \sum_{i=1}^m \updelta_{Z_i}$, where $Z_{1:m}$ are \iid~samples from $\mu$. Then, the map $(\nu, \mu) \mapsto \expe{ \swassersteinD[p](\nu, \hat{\mu}_m)}$ is lower semi-continuous on $\calP_p(\msy) \times \calP_p(\msy)$ endowed with the topology of weak convergence.
\end{lemma}

\begin{proof}
  We consider two sequences $(\mu_k)_{k \in \bbN}$ and $(\nu_k)_{k \in \bbN}$ of probability measures in $\msy$, such that $\sequencek{\mu_k} \wc \mu$ and $\sequencek{\nu_k} \wc \nu$, and we fix $m \in \bbN^*$. 

  By Skorokhod's representation theorem, there exists a probability space $(\tilde{\Omega},\tilde{\mcf},\tilde{\PP})$, a sequence of random variables $(\tilde{X}_k^{1},\ldots,\tilde{X}_k^{m})_{k \in \nset}$ and a random variable $(\tilde{X}^1,\ldots,\tilde{X}^m)$ defined on $\tilde{\Omega}$ such that for any $k \in \nset$ and $i \in \{1,\ldots,m\}$, $\tilde{X}_k^i$ has distribution $\mu_k$, $\tilde{X}^i$ has distribution $\mu$ and $(\tilde{X}_k^{1},\ldots,\tilde{X}_k^{m})_{k \in \nsets}$ converges  to $(\tilde{X}^1,\ldots,\tilde{X}^m)$, $\tilde{\PP}$-almost surely.
We then show that the sequence of (random) empirical distributions $(\hmu_{k,m})_{k \in \nset}$ defined by  $\hmu_{k,m} = (1/m) \sum_{i=1}^m \updelta_{\tilde{X}^{i}_k}$, weakly converges to $\hmu_{m} = (1/m) \sum_{i=1}^m \updelta_{\tilde{X}^i}$, $\tilde{\PP}$-almost surely. Note that it is sufficient to show that for any deterministic sequence $(x_k^{1},\ldots,x_k^{m})_{k \in \nsets}$ which converges  to $(x^1,\ldots,x^m)$, \ie~$\lim_{k \to \plusinfty} \max_{i\in \{1,\ldots,m\}} \rho(x_k^i,x^i) = 0$, then 
the sequence of empirical distributions $(\hnu_{k,m})_{k \in \nset}$ defined by  $\hnu_{k,m} = (1/m) \sum_{i=1}^m \updelta_{x^i_k}$, weakly converges to $\hnu_{m} = (1/m) \sum_{i=1}^m \updelta_{x^i}$. Note that since the Lévy-Prokhorov metric $\dist_{\mathcal{P}}$ metrizes the weak convergence by \cite[Theorem 6.8]{Billingsley1999}, we only need to show that $\lim_{k \to \plusinfty} \dist_{\mathcal{P}}(\hnu_{k,m},\hnu_m) = 0$. More precisely, since for any probability measure $\zeta_1$ and $\zeta_2$,
\begin{equation*}
  \dist_{\mcp}(\zeta_1,\zeta_2) = \inf \defEns{ \epsilon >0 \, : \, \text{ for any $\msa \in \mcy$, } \zeta_1(\msa) \leq \zeta_2(\msa^{\epsilon}) + \epsilon \text{ and }  \zeta_2(\msa) \leq \zeta_1(\msa^{\epsilon}) + \epsilon} \eqsp, 
\end{equation*}
where $\mcy$ is the Borel $\sigma$-field of $(\msy,\rho)$ and for any $\msa \in \mcy$, $\msa^{\epsilon} = \{x \in \msy \, : \, \rho(x,y) < \epsilon \text{ for any } y \in \msa\}$, we get
\begin{equation*}
\dist_{\mcp}(\hnu_{k,m},\hnu_{m})  \leq 2 \max_{i\in \{1,\ldots,m\}} \rho(x_k^i,x^i) \eqsp,
\end{equation*}
and therefore $\lim_{k \to \plusinfty} \dist_{\mcp}(\hnu_{k,m},\hnu_{m}) = 0$, so that, $\sequencek{\hnu_{k,m}}$ weakly converges to $\hnu_m$.

Finally, we have that $\hmu_{k,m} = (1/m) \sum_{i=1}^m \updelta_{\tilde{X}^{i}_k}$, weakly converges to $\hmu_{m} = (1/m) \sum_{i=1}^m \updelta_{\tilde{X}^i}$, $\tilde{\PP}$-almost surely and we  obtain the final result using the lower semi-continuity of the Sliced-Wasserstein distance derived in \Cref{lem:sw_semicontinuous} and Fatou's lemma which give
  \begin{equation*}
    \tilde{\PE}\parentheseDeux{ \swassersteinD[p](\nu, \hat{\mu}_m) } \leq     \tilde{\PE}\parentheseDeux{ \liminf_{i \rightarrow \infty} \swassersteinD[p](\nu_i, \hat{\mu}_{m,i}) } \leq \liminf_{i \rightarrow \infty}     \tilde{\PE}\parentheseDeux{ \{ \swassersteinD[p](\nu_i, \hat{\mu}_{m,i}) } \eqsp,
  \end{equation*}
  where $\tilde{\PE}$ is the expectation corresponding to $\tilde{\PP}$.
\end{proof}

The following corollary is a direct consequence of Lemma~\ref{lemma:lsc_Esw}.

\begin{corollary} \label{coro:esw_semicontinuous_2}
  Assume \cref{assumption:continuousmap}. Then, $(\nu, \theta) \mapsto \expeYLigne{ \swassersteinD[p](\nu, \hat{\mu}_{\theta,m}) }$ is lower semi-continuous on $\mcp(\msy) \times \Theta$.
\end{corollary}

\subsection{Measurability of the MSWE and MESWE}\label{subsec:measurability}

The measurability of the MSWE and MESWE follows from the application of \cite[Corollary 1]{brown1973}, also used in \cite{bassetti2006} and \cite{Bernton2019}, and which we recall in \Cref{thm:corollary1}.

\begin{theorem}[Corollary 1 in \cite{brown1973}] \label{thm:corollary1} Let $\msu,\msv$ be Polish spaces and $f$ be a real-valued Borel measurable function defined on a Borel subset $\msd$ of $\msu \times \msv$. We denote by $\projD$ the set defined as
  \begin{equation*}
    \projD = \{ u\ :\ \text{there exists } v \in \msv,\ (u,v) \in \msd \} \eqsp.
  \end{equation*}
  Suppose that for each $u \in \projD$, the section $\msd_u = \{ v \in V, (u,v) \in \msd \}$ is $\sigma$-compact and $f(u, \cdot)$ is lower semi-continuous with respect to the relative topology on $\msd_u$. Then, 
\begin{enumerate}
    \item The sets $\projD$ and $\msi = \{ u \in \projD,\ \text{for some } v \in \msd_u,\ f(u,v) = \inf f_u \}$ are Borel
    \item For each $\epsilon > 0$, there is a Borel measurable function $\phi_\epsilon$ satisfying, for $u \in \projD$, 
    \begin{align*}
        f(u, \phi_\epsilon(u)) &= \inf_{\msd_u} f_u, &\text{if }\; u \in \msi, \;\;\;& \\
        &\leq \epsilon + \inf_{\msd_u} f_u, &\text{if }\; u \notin \msi, \;\;\;&\text{and }\;\; \inf_{\msd_u} f_u \neq -\infty \\
        &\leq - \epsilon^{-1}, &\text{if }\; u \notin \msi, \;\;\;&\text{and }\;\; \inf_{\msd_u} f_u = -\infty \eqsp.
    \end{align*}
\end{enumerate}
\end{theorem}

\begin{theorem}[Measurability of the MSWE] \label{thm:measurability} Assume \cref{assumption:continuousmap}. For any $n \geq 1$ and $\epsilon > 0$, there exists a Borel measurable function $\hat{\theta}_{n,\epsilon} : \Omega \rightarrow \Theta$ that satisfies: for any $\omega \in \Omega$,
\begin{equation*}
    \hat{\theta}_{n,\epsilon}(\omega) \in \left\{
    \begin{array}{ll}
        \argmin_{\theta \in \Theta} \;\;\; \swassersteinD[p](\hat{\mu}_n(\omega), \mu_\theta), & \;\; \text{if this set is non-empty,} \\
        \{ \theta \in \Theta\ :\ \swassersteinD[p](\hat{\mu}_n(\omega), \mu_\theta) \leq \epsilon_\star + \epsilon \}, & \;\; \text{otherwise.}
    \end{array}
    \right.
\end{equation*}
where $\epsilon_\star = \inf_{\theta \in \Theta} \swassersteinD[p](\mu_\star, \mu_\theta)$.
\end{theorem}
\begin{proof}
    The proof consists in showing that the conditions of Theorem~\ref{thm:corollary1} are satisfied. 

    The empirical measure $\hat{\mu}_n(\omega)$ depends on $\omega \in \Omega$ only through $y = (y_1,  \dots, y_n) \in \msy^n$, so we can consider it as a function on $\msy^n$ rather than on $\Omega$. We introduce $\msd = \msy^n \times \Theta$. Since $\msy$ is Polish, $\msy^n$ ($n \in \mathbb{N}^*$) endowed with the product topology is Polish. For any $y \in \msy^n$, the set $\msd_y = \{ \theta \in \Theta,\ (y, \theta) \in \msd \} = \Theta$ is assumed to be $\sigma$-compact.
    
    The map $y \mapsto \hat{\mu}_n(y)$ is continuous for the weak topology (see the proof of \Cref{lemma:lsc_Esw}), as well as the map $\theta \mapsto \mu_\theta$ according to \cref{assumption:continuousmap}. We deduce by Corollary~\ref{coro:sw_semicontinuous_2} that the map $(\mu, \theta) \mapsto \swassersteinD[p](\mu, \mu_\theta)$ is l.s.c. for the weak topology. Since the composition of a lower semi-continuous function with a continuous function is l.s.c., the map $(y, \theta) \mapsto \swassersteinD[p](\hat{\mu}_n(y), \mu_\theta)$ is l.s.c. for the weak topology, thus measurable and for any $y \in \msy^n$, $\theta \mapsto \swassersteinD[p](\hat{\mu}_n(y), \mu_\theta)$ is l.s.c. on $\Theta$. A direct application of Theorem~\ref{thm:corollary1} finalizes the proof.

\end{proof}

\begin{theorem}[Measurability of the MESWE] \label{thm:measurability_meswe} Assume \cref{assumption:continuousmap}. For any $n \geq 1$, $m \geq 1$ and $\epsilon > 0$, there exists a Borel measurable function $\hat{\theta}_{n,m,\epsilon} : \Omega \rightarrow \Theta$ that satisfies: for any $\omega \in \Omega$,
\begin{equation*}
    \hat{\theta}_{n,m,\epsilon}(\omega) \in \left\{
    \begin{array}{ll}
        \argmin_{\theta \in \Theta} \;\;\; \expe{ \swassersteinD[p](\hat{\mu}_n(\omega), \hat{\mu}_{\theta,m}) \middle| Y_{1:n} }, & \; \text{if this set is non-empty,} \\
        \big\{ \theta \in \Theta\ :\ \expe{ \swassersteinD[p](\hat{\mu}_n(\omega), \hat{\mu}_{\theta,m}) \middle| Y_{1:n} } \leq \epsilon_* + \epsilon \} \big\}, & \; \text{otherwise.}
    \end{array}
    \right.
\end{equation*}
where $\epsilon_* = \inf_{\theta \in \Theta} \expeLigne{ \swassersteinD[p](\hat{\mu}_n(\omega), \hat{\mu}_{\theta,m}) | Y_{1:n}}$.
\end{theorem}

\begin{proof}%
  The proof can be done similarly to the proof of \Cref{thm:measurability}: we verify that we can apply \Cref{thm:corollary1} using \Cref{coro:esw_semicontinuous_2} instead of \Cref{coro:sw_semicontinuous_2}.
\end{proof}

\section{Postponed proofs}
\label{sec:postponed-proofs}

\subsection{Proof of \Cref{thm:SWp_metrizes_Pp}}

\begin{lemma} \label{lem:cvg_sw1_implies_wc}
    Let $\sequencek{\mu_k}$ be a sequence of probability measures on $\mathbb{R}^d$ and $\mu$ a measure in $\mathbb{R}^d$ such that, 
      \begin{equation*}
        \lim_{k \rightarrow \infty} \swassersteinD[1](\mu_k, \mu) = 0 \eqsp.
      \end{equation*}
    Then, there exists an increasing function $\phi : \mathbb{N} \rightarrow \mathbb{N}$ such that the subsequence $\sequencek{\mu_{\phi(k)}}$ converges weakly to $\mu$.
\end{lemma}

\begin{proof}
  By definition, we have that
  \begin{equation*}
    \lim_{k \rightarrow \infty} \int_{\sphereD} \wassersteinD[1](\usss \mu_k, \usss \mu) \rmd\unifS(u) = 0 \eqsp. 
  \end{equation*}

  Therefore by \cite[Theorem 2.2.5]{Bogachev2007}, there exists an increasing mapping $\phi : \mathbb{N} \rightarrow \mathbb{N}$ such that for $\unifS$-almost every ($\unifS$-a.e.) $u \in \sphereD$, $\lim_{k \rightarrow \infty} \wassersteinD[1](\usss \mu_{\phi(k)}, \usss \mu) = 0 $ . By \cite[Theorem 6.9]{Villani2008}, it implies that for $\unifS$-a.e. $u \in \sphereD$, $\sequencek{\usss \mu_{\phi(k)}} \wc \usss \mu$. L\'evy's characterization \cite[Theorem 4.3]{kallenberg:1997} gives that, for $\unifS$-a.e. $u \in \sphereD$ and any $s \in \mathbb{R}$, 
  \begin{equation*}
    \lim_{k \rightarrow \infty} \Phi_{\usss \mu_{\phi(k)}}(s) = \Phi_{\usss \mu}(s) \eqsp,
  \end{equation*}
  where, for any distribution $\nu \in \mcp(\rset^p)$, $\Phi_\nu$ denotes the characteristic function of $\nu$ and is defined for any $v \in \mathbb{R}^p$ as
  \begin{equation*}
    \Phi_\nu(v) = \int_{\mathbb{R}^p} \rme^{\rmi \langle v, w \rangle} \rmd\nu(w)  \eqsp.
  \end{equation*}

  Then, we can conclude that for Lebesgue-almost every $z \in \mathbb{R}^d$,
  \begin{equation} \label{eqn:cvg_characteristic}
    \lim_{k \rightarrow \infty} \Phi_{\mu_{\phi(k)}}(z) = \Phi_{\mu}(z)  \eqsp. 
  \end{equation}

  We can now show that $\sequencek{\mu_{\phi(k)}} \wc \mu$, \ie~by \cite[Problem 1.11, Chapter 1]{Billingsley1999} for any $f: \rset^d \to \rset$ continuous with compact support, 
  \begin{equation}
    \label{eq:convo_0}
    \lim_{n \rightarrow \infty} \int_{\mathbb{R}^d} f(z) \rmd\mu_n(z) = \int_{\mathbb{R}^d} f(z) \rmd\mu(z) \eqsp. 
  \end{equation}

  Let $f : \rset^d \to \rset$ be a continuous function with compact support  and $\sigma > 0$. Consider the function $f_\sigma$ defined for any $x \in \mathbb{R}^d$ as
  \begin{equation*}
    f_\sigma(x) = (2\uppi \sigma^2)^{-d/2}  \int_{\mathbb{R}^d} f(x-z) \exp\parenthese{-\|z\|^2/{2\sigma^2}} \rmd\Leb(z) = f \ast g_\sigma(x) \eqsp,
  \end{equation*}
  where $g_\sigma$ is the density of the $d$-dimensional Gaussian with zero mean and covariance matrix $\sigma^2 \bfI_d$, and $\ast$ denotes the convolution product. 

We first show that \eqref{eq:convo_0} holds with $f_{\sigma}$ in place of $f$. Since for any $z \in \rset^d$, $\expe{\rme^{\rmi\ps{G}{z}}} = \rme^{\rmi \ps{\mathtt{m}}{z} + (1/(2\sigma^2))\norm[2]{z}}$ if $G$ is a $d$-dimensional Gaussian random variable with zero mean and covariance matrix $(1/\sigma^2) \Idd$, by Fubini's theorem we get  for any $k \in \nset$
  \begin{align}
\nonumber
    \int_{\mathbb{R}^d} f_\sigma(z) \rmd\mu_{\phi(k)}(z) &= \int_{\mathbb{R}^d} \int_{\mathbb{R}^d} f(w) g_\sigma(z-w) \rmd w \rmd \mu_{\phi(k)}(z) \\
    \nonumber
                                                         &= \int_{\mathbb{R}^d} \int_{\mathbb{R}^d} f(w) (2\uppi \sigma^2)^{-d/2} \int_{\mathbb{R}^d} \rme^{i \langle z-w, x \rangle}  g_{1/\sigma}(x) \rmd x \rmd w \rmd \mu_{\phi(k)}(z) \\
    \nonumber
                                                         &= \int_{\mathbb{R}^d} \int_{\mathbb{R}^d} (2\uppi \sigma^2)^{-d/2} f(w) \rme^{-i\langle w,x \rangle } g_{1/\sigma}(x) \Phi_{\mu_{\phi(k)}}(x)  \rmd x \rmd w \\
    \label{eq:convo_1}
    &= (2\uppi \sigma^2)^{-d/2} \int_{\mathbb{R}^d}   \mathcal{F}[f](x) g_{1/\sigma}(x) \Phi_{\mu_{\phi(k)}}(x) \rmd x \eqsp,
  \end{align}
  where $\mathcal{F}[f](x) = \int_{\rset^d} f(w) \rme^{\rmi \ps{w}{x}} \rmd w$ denotes the Fourier transform of $f$, which exists since $f$ is assumed to have a compact support. In an analogous manner, we prove that 
  \begin{equation}
    \label{eq:convo_2}
    \int_{\mathbb{R}^d} f_\sigma(z) \rmd \mu(z) =  (2\uppi \sigma^2)^{-d/2} \int_{\mathbb{R}^d}  \mathcal{F}[f](x) g_{1/\sigma}(x) \Phi_{\mu}(x) \rmd x \eqsp.
  \end{equation}

  Now, using that $\mathcal{F}[f]$ is bounded by $\int_{\mathbb{R}^d} |f(w)| \rmd w < \plusinfty$ since $f$ has compact support, we obtain that, for any $k \in \mathbb{N}$ and $x \in \mathbb{R}^d$,
  \begin{equation*}
    \left| \mathcal{F}[f](x) g_{1/\sigma}(x) \Phi_{\mu_{\phi(k)}}(x) \right| \leq  g_{1/\sigma}(x) \int_{\mathbb{R}^d} |f(w)| \rmd w
  \end{equation*}

  By \eqref{eqn:cvg_characteristic}, \eqref{eq:convo_1}, \eqref{eq:convo_2} and Lebesgue's Dominated Convergence Theorem, we obtain
  \begin{align} 
    \lim_{k \rightarrow \infty} \int_{\mathbb{R}^d} (2\uppi \sigma^2)^{-d/2} \mathcal{F}[f](x) g_{1/\sigma}(x) \Phi_{\mu_{\phi(k)}}(x) \rmd x &= \int_{\mathbb{R}^d} (2\uppi \sigma^2)^{-d/2}  \mathcal{F}[f](x) g_{1/\sigma}(x) \Phi_\mu(x) \rmd x \nonumber \\
 \lim_{k \rightarrow \infty} \int_{\mathbb{R}^d} f_{\sigma}(z) \rmd \mu_{\phi(k)}(z) &= \int_{\mathbb{R}^d} f_\sigma(z) \rmd\mu(z) \eqsp. \label{eqn:cvg_l1_f_sigma} 
  \end{align}

We can now complete the proof of \eqref{eq:convo_0}. For any $\sigma > 0$,  we have 
  \begin{multline*}
    \left| \int_{\mathbb{R}^d} f(z) \rmd \mu_{\phi(k)}(z) - \int_{\mathbb{R}^d} f(z) \rmd \mu(z) \right| \leq 2\sup_{z \in \mathbb{R}^d} \left| f(z) - f_\sigma(z) \right| \\
    + \left| \int_{\mathbb{R}^d} f_\sigma(z) \rmd \mu_{\phi(k)}(z) - \int_{\mathbb{R}^d} f_{\sigma}(z) \rmd \mu(z) \right| \eqsp.     %
  \end{multline*}
  Therefore by \eqref{eqn:cvg_l1_f_sigma}, for any $\sigma >0$, we get
  \begin{equation*}
    \limsup_{k \to \plusinfty} \left| \int_{\mathbb{R}^d} f(z) \rmd \mu_{\phi(k)}(z) - \int_{\mathbb{R}^d} f(z) \rmd \mu(z) \right|\\ \leq 2\sup_{z \in \mathbb{R}^d} \left| f(z) - f_\sigma(z) \right| \eqsp.
  \end{equation*}
  Finally \cite[Theorem 8.14-b]{folland:1999} implies that  $\lim_{\sigma \rightarrow 0} \sup_{z \in \mathbb{R}^d} | f_\sigma(z) - f(z) | = 0$ which  concludes the proof.
  \end{proof}

  \begin{proof}[Proof of Theorem~\ref{thm:SWp_metrizes_Pp}]

    Now, assume that
    \begin{equation}
      \label{eqn:assum_sw_lim_zero}
\lim_{k \rightarrow \infty} \swassersteinD[p](\mu_k, \mu) = 0      
    \end{equation}
 and that $(\mu_k)_{k \in \nset}$ does not converge weakly to $\mu$. Therefore, $\lim_{k \rightarrow \infty} \dist_{\mcp}(\mu_k, \mu) \neq 0$, where $\dist_{\mcp}$ denotes the Lévy-Prokhorov metric, and there exists $\epsilon > 0$ and a subsequence $\sequencek{\mu_{\psi(k)}}$ with $\psi : \bbN \rightarrow \bbN$ increasing, such that for any $k \in\nset$,
    \begin{equation} \label{eqn:lim_not_zero}
\dist_{\mcp}(\mu_{\psi(k)}, \mu) > \epsilon
    \end{equation}

    In addition, by Hölder's inequality, we know that $\wassersteinD[1](\mu_k, \mu) \leq \wassersteinD[p](\mu_k, \mu)$, thus $\swassersteinD[1](\mu_k, \mu) \leq \swassersteinD[p](\mu_k, \mu)$, and by \eqref{eqn:assum_sw_lim_zero}, $\lim_{k \rightarrow \infty} \swassersteinD[1](\mu_{\psi(k)}, \mu) = 0$. Then, according to Lemma~\ref{lem:cvg_sw1_implies_wc}, there exists a subsequence $\sequencek{\mu_{\phi(\psi(k))}}$ with $\phi : \bbN \rightarrow \bbN$ increasing, such that
    \begin{equation*}
      \mu_{\phi(\psi(k))} \wc \mu  
    \end{equation*}
    which is equivalent to $\lim_{k \rightarrow \infty} \dist_{\mcp}(\mu_{\phi(\psi(k))}, \mu) = 0$, thus contradicts \eqref{eqn:lim_not_zero}. We conclude that \eqref{eqn:assum_sw_lim_zero} implies $\sequencek{\mu_k} \wc \mu$.
\end{proof}

\subsection{Minimum Sliced-Wasserstein estimators: Proof of \Cref{thm:existence_consistency_mswe}}

\begin{proof}[Proof of \Cref{thm:existence_consistency_mswe}]
    This result is proved analogously to the proof of Theorem 2.1 in \cite{Bernton2019}. The key step is to show that the function $\theta \mapsto \swassersteinD[p](\hat{\mu}_n, \mu_\theta)$ epi-converges to $\theta \mapsto \swassersteinD[p](\mu_\star, \mu_\theta)$ $\mathbb{P}$-almost surely, and then apply Theorem 7.31 of \cite{Rockafellar2009} (recalled in Theorem~\ref{thm:731rockafellar}).
    
    First, by \cref{assumption:continuousmap} and \Cref{coro:sw_semicontinuous_2}, the map $\theta \mapsto \swassersteinD[p](\mu, \mu_\theta)$ is l.s.c. on $\Theta$ for any $\mu \in \mathcal{P}_p(\msy)$. Therefore by \Cref{assumption:boundedset}, there exists $\theta_\star \in \Theta$ such that $\swassersteinD[p](\mu_\star, \mu_{\theta_\star}) = \epsilon_\star$ and the set $\Theta^\star_\epsilon$ %
    is non-empty as it contains $\theta_\star$, closed by lower semi-continuity of $\theta \mapsto \swassersteinD[p](\mu_\star, \mu_\theta)$, and bounded. $\Theta^\star_\epsilon$ is thus compact, and we conclude again by lower semi-continuity that the set $\argmin_{\theta \in \Theta} \swassersteinD[p](\mu_\star, \mu_\theta)$ is non-empty \cite[Theorem 2.43]{aliprantis1999infinite}. 
    
Consider the event given by \Cref{assumption:datagen},  $\mse \in \mcf$ such that $\PP(\mse) = 1$ and for any $\omega \in \mse$, $\lim_{n \rightarrow \infty} \swassersteinD[p](\hat{\mu}_n(\omega), \mu_\star) = 0$.   Then, we prove that $\theta \mapsto \swassersteinD[p](\hat{\mu}_n, \mu_\theta)$ epi-converges to $\theta \mapsto \swassersteinD[p](\mu_\star, \mu_\theta)$ $\mathbb{P}$-almost surely using the characterization in \cite[Proposition 7.29]{Rockafellar2009}, \ie~we verify that, for any $\omega \in \mse$, the two conditions below hold: $\text{for every compact set } \msk \subset \Theta$ and every open set  $\mso \subset \Theta$,
    \begin{equation}
      \begin{aligned}
        \liminf_{n \rightarrow \infty} \inf_{\theta \in \msk} \swassersteinD[p](\hat{\mu}_n(\omega), \mu_\theta)& \geq \inf_{\theta \in \msk} \swassersteinD[p](\mu_\star, \mu_\theta)   \\
        \limsup_{n \rightarrow \infty} \inf_{\theta \in \mso} \swassersteinD[p](\hat{\mu}_n(\omega), \mu_\theta)& \leq \inf_{\theta \in \mso} \swassersteinD[p](\mu_\star, \mu_\theta) \eqsp. 
      \end{aligned}
      \label{eqn:mswe_epiconv_cond}
    \end{equation}
    
    We fix $\omega$ in $\mse$. Let $\msk \subset \Theta$ be a compact set. By lower semi-continuity of $\theta \mapsto \swassersteinD[p](\hat{\mu}_n(\omega), \mu_\theta)$, there exists $\theta_n = \theta_n(\omega) \in \msk$ such that for any $n \in \mathbb{N}$, $\inf_{\theta \in \msk} \swassersteinD[p](\hat{\mu}_n(\omega), \mu_\theta) = \swassersteinD[p](\hat{\mu}_n(\omega), \mu_{\theta_n})$. 

    We consider the subsequence $\sequencen{\hat{\mu}_{\phi(n)}}$ where $\phi : \mathbb{N} \rightarrow \mathbb{N}$ is increasing such that $\swassersteinD[p](\hat{\mu}_{\phi(n)}(\omega), \mu_{\theta_{\phi(n)}})$ converges to $\liminf_{n \rightarrow \infty} \swassersteinD[p](\hat{\mu}_n(\omega), \mu_{\theta_n}) = \liminf_{n \rightarrow \infty} \inf_{\theta \in \msk} \swassersteinD[p](\hat{\mu}_n(\omega), \mu_\theta)$. Since $\msk$ is compact, there also exists an increasing function $\psi : \mathbb{N} \rightarrow \mathbb{N}$ such that, for $\bar{\theta} \in \msk$, $\lim_{n \rightarrow \infty } \rho_\Theta(\theta_{\psi(\phi(n))}, \bar{\theta}) = 0$. Therefore, we have  
    \begin{align}
        \liminf_{n \rightarrow \infty} \inf_{\theta \in \msk} \swassersteinD[p](\hat{\mu}_n(\omega), \mu_\theta) &= \lim_{n \rightarrow \infty} \swassersteinD[p](\hat{\mu}_{\phi(n)}(\omega), \mu_{\theta_{\phi(n)}}) \nonumber \\
        &= \lim_{n \rightarrow \infty} \swassersteinD[p](\hat{\mu}_{\psi(\phi(n))}(\omega), \mu_{\theta_{\psi(\phi(n))}}) \nonumber \\
        &= \liminf_{n \rightarrow \infty} \swassersteinD[p](\hat{\mu}_{\psi(\phi(n))}(\omega), \mu_{\theta_{\psi(\phi(n))}}) \nonumber \\
        &\geq \swassersteinD[p](\mu_\star, \mu_{\bar{\theta}}) \label{eq:thm1_lsc} \\
        &\geq \inf_{\theta \in \msk} \swassersteinD[p](\mu_\star, \mu_\theta) \eqsp, \nonumber
    \end{align}
    where \eqref{eq:thm1_lsc} is obtained by lower semi-continuity since $\hat{\mu}_{\psi(\phi(n))}(\omega) \wc \mu_\star$ by \cref{assumption:datagen} and Theorem~\ref{thm:SWp_metrizes_Pp}, and $\mu_{\theta_{\psi(\phi(n))}} \wc \mu_{\bar{\theta}}$ by \cref{assumption:continuousmap}. We conclude that the first condition in \eqref{eqn:mswe_epiconv_cond} holds.
    
    Now, we fix $\mso \subset \Theta$ open. By definition of the infimum, there exists a sequence $\sequencen{\theta_n}$ in $\mso$ such that $\{\swassersteinD[p](\mu_\star, \mu_{\theta_n})\}_{n \in \nset}$ converges to $\inf_{\theta \in \mso} \swassersteinD[p](\mu_\star, \mu_\theta)$. For any $n \in \mathbb{N}$, $\inf_{\theta \in \mso} \swassersteinD[p](\hat{\mu}_n(\omega), \mu_\theta) \leq \swassersteinD[p](\hat{\mu}_n(\omega), \mu_{\theta_n})$. Therefore,
    \begin{align*}
        &\limsup_{n \rightarrow \infty} \inf_{\theta \in \mso} \swassersteinD[p](\hat{\mu}_n(\omega), \mu_\theta) \leq \limsup_{n \rightarrow \infty} \swassersteinD[p](\hat{\mu}_n(\omega), \mu_{\theta_n}) \nonumber \\
        & \qquad \qquad \qquad\leq \limsup_{n \rightarrow \infty} \big( \swassersteinD[p](\hat{\mu}_n(\omega), \mu_\star) + \swassersteinD[p](\mu_\star, \mu_{\theta_n}) \big) \; \text{by the triangle inequality} \\
        &\qquad \qquad\qquad\leq  \limsup_{n \rightarrow \infty} \swassersteinD[p](\mu_\star, \mu_{\theta_n}) \text{ by \cref{assumption:datagen}} \\
        &\qquad \qquad\qquad = \inf_{\theta \in \mso} \swassersteinD[p](\mu_\star, \mu_\theta) \;\; \text{by definition of $\sequencen{\theta_n}$} \eqsp.
    \end{align*}
    This shows that the second condition in \eqref{eqn:mswe_epiconv_cond} holds, and hence, the sequence of functions $\theta \mapsto \swassersteinD[p](\hat{\mu}_n(\omega), \mu_\theta)$ epi-converges to $\theta \mapsto \swassersteinD[p](\mu_\star, \mu_\theta)$. 

    Now, we apply Theorem 7.31 of \cite{Rockafellar2009}. First, by \cite[Theorem 7.31(b)]{Rockafellar2009}, \eqref{eqn:mswe_consist_b} immediately follows from the epi-convergence of $\theta \mapsto \swassersteinD[p](\hat{\mu}_n(\omega), \mu_\theta)$ to $\theta \mapsto \swassersteinD[p](\mu_\star, \mu_\theta)$.

    Next, we show that \cite[Theorem 7.31(a)]{Rockafellar2009} can be applied showing that  for any $\eta > 0$ there exists a compact set $\msb \subset \Theta$ and $N \in \nset$ such that, for all $n \geq N$, 
    \begin{equation}
\label{eqn:mswe_cond_for_thm731a_0}
      \inf_{\theta \in \msb} \swassersteinD[p](\hat{\mu}_n(\omega), \mu_\theta) \leq \inf_{\theta \in \Theta} \swassersteinD[p](\hat{\mu}_n(\omega), \mu_\theta) + \eta \eqsp.
    \end{equation}

    In fact, we simply show that there exists a compact set $\msb \subset \Theta$ and $N \in \nset$ such that, for all $n \geq N$, $\inf_{\theta \in \msb} \swassersteinD[p](\hat{\mu}_n(\omega), \mu_\theta) = \inf_{\theta \in \Theta} \swassersteinD[p](\hat{\mu}_n(\omega), \mu_\theta)$.

    On one hand, the second condition in \eqref{eqn:mswe_epiconv_cond} gives us 
    \begin{equation*}
      \limsup_{n \rightarrow \infty} \inf_{\theta \in \Theta} \swassersteinD[p](\hat{\mu}_n(\omega), \mu_\theta) \leq \inf_{\theta \in \Theta} \swassersteinD[p](\mu_\star, \mu_\theta) = \epsilon_\star \eqsp.
    \end{equation*} 

    We deduce that there exists $n_{\epsilon/4}(\omega)$ such that, for $n \geq n_{\epsilon/4}(\omega)$,  $      \inf_{\theta \in \Theta} \swassersteinD[p](\hat{\mu}_n(\omega), \mu_\theta) \leq \epsilon_\star + \epsilon/4$,
where $\epsilon$ is given by \Cref{assumption:boundedset}. As $n \geq n_{\epsilon/4}(\omega)$, the set $\widehat{\Theta}_{\epsilon/2} = \{ \theta \in \Theta : \swassersteinD[p](\hat{\mu}_n(\omega), \mu_\theta) \leq \epsilon_\star + \frac{\epsilon}{2} \}$ is non-empty as it contains $\theta^*$ defined as $\swassersteinD[p](\hat{\mu}_n(\omega), \mu_{\theta^*}) = \inf_{\theta \in \Theta} \swassersteinD[p](\hat{\mu}_n(\omega), \mu_\theta)$. 

    On the other hand, by \cref{assumption:datagen}, there exists $n_{\epsilon/2}(\omega)$ such that, for $n \geq n_{\epsilon/2}(\omega)$, 
    \begin{equation} \label{eqn:mswe_n_eps2}
      \swassersteinD[p](\hat{\mu}_n(\omega), \mu_\star) \leq \frac{\epsilon}{2} \eqsp.
    \end{equation}

    Let $n \geq n_*(\omega) = \max\{n_{\epsilon/4}(\omega), n_{\epsilon/2}(\omega) \}$ and $\theta \in \widehat{\Theta}_{\epsilon/2}$. By the triangle inequality, 
    \begin{align*}
      \swassersteinD[p](\mu_\star, \mu_\theta) &\leq \swassersteinD[p](\hat{\mu}_n(\omega), \mu_\star) + \swassersteinD[p](\hat{\mu}_n(\omega), \mu_\theta) \\
      &\leq \epsilon_\star + \epsilon \text{\;\;\; since $\theta \in \widehat{\Theta}_{\epsilon/2}$ and by \eqref{eqn:mswe_n_eps2}}  
    \end{align*}

    This means that, when $n \geq n_*(\omega)$, $\widehat{\Theta}_{\epsilon/2} \subset \Theta^\star_\epsilon$, and since $\inf_{\theta \in \Theta} \swassersteinD[p](\hat{\mu}_n(\omega), \mu_\theta)$ is attained in $\widehat{\Theta}_{\epsilon/2}$, we have
    \begin{equation}
      \inf_{\theta \in \Theta^\star_\epsilon} \swassersteinD[p](\hat{\mu}_n(\omega), \mu_\theta) = \inf_{\theta \in \Theta} \swassersteinD[p](\hat{\mu}_n(\omega), \mu_\theta) \eqsp. \label{eqn:mswe_cond_for_thm731a}
    \end{equation}
As shown in the first part of the proof $\Theta^{\star}_{\epsilon}$ is compact and then by \cite[Theorem 7.31(a)]{Rockafellar2009}, \eqref{eqn:mswe_consist_a} is a direct consequence of \eqref{eqn:mswe_cond_for_thm731a_0}-\eqref{eqn:mswe_cond_for_thm731a} and the epi-convergence of $\theta \mapsto \swassersteinD[p](\hat{\mu}_n(\omega), \mu_\theta)$ to $\theta \mapsto \swassersteinD[p](\mu_\star, \mu_\theta)$. 

    Finally, by the same reasoning that was done earlier in this proof for $\argmin_{\theta \in \Theta} \swassersteinD[p](\mu_\star, \mu_\theta)$, the set $\argmin_{\theta \in \Theta} \swassersteinD[p](\hat{\mu}_n(\omega), \mu_\theta)$ is non-empty for $n \geq n_*(\omega)$. 

\end{proof}

\subsection{Existence and consistency of the MESWE: Proof of \Cref{thm:existence_consistency_meswe}}

\begin{proof}[Proof of \Cref{thm:existence_consistency_meswe}]

This result is proved analogously to the proof of \cite[Theorem 2.4]{Bernton2019}. The key step is to show that the function $\theta \mapsto \expeYLigne{ \swassersteinD[p](\hat{\mu}_n, \hat{\mu}_{\theta, m(n)}) }$ epi-converges to $\theta \mapsto \expeYLigne{ \swassersteinD[p](\mu_\star, \mu_\theta) }$, and then apply \cite[Theorem 7.31]{Rockafellar2009}, which we recall in Theorem~\ref{thm:731rockafellar}.

First, since we assume \cref{assumption:continuousmap} and \cref{assumption:boundedset}, we can apply the same reasoning as in the proof of Theorem~\ref{thm:existence_consistency_mswe} to show that the set $\argmin_{\theta \in \Theta} \swassersteinD[p](\mu_\star, \mu_\theta)$ is non-empty.

Consider the event given by \Cref{assumption:datagen},  $\mse \in \mcf$ such that $\PP(\mse) = 1$ and for any $\omega \in \mse$, $\lim_{n \rightarrow \infty} \swassersteinD[p](\hat{\mu}_n(\omega), \mu_\star) = 0$. Then, we prove that $\theta \mapsto \expeYLigne{ \swassersteinD[p](\hat{\mu}_n, \hat{\mu}_{\theta, m(n)}) }$ epi-converges to $\theta \mapsto \swassersteinD[p](\mu_\star, \mu_\theta)$ $\mathbb{P}$-almost surely using the characterization of \cite[Proposition 7.29]{Rockafellar2009}, \ie~we verify that, for any $\omega \in \mse$, the two conditions below hold: for every compact set $\msk \subset \Theta$ and for every open set $\mso \subset \Theta$,
\begin{equation}
  \begin{aligned}
        \liminf_{n \rightarrow \plusinfty} \inf_{\theta \in \msk} \expeY{ \swassersteinD[p](\hat{\mu}_n(\omega), \hat{\mu}_{\theta, m(n)})  } \geq \inf_{\theta \in \msk} \swassersteinD[p](\mu_\star, \mu_\theta) \\
        \limsup_{n \rightarrow \plusinfty} \inf_{\theta \in \mso} \expeY{ \swassersteinD[p](\hat{\mu}_n(\omega), \hat{\mu}_{\theta, m(n)})  } \leq \inf_{\theta \in \mso} \swassersteinD[p](\mu_\star, \mu_\theta)
  \end{aligned}
      \label{eqn:meswe_epiconv_cond}
\end{equation}

We fix $\omega$ in $\mse$. Let $\msk \subset \Theta$ be a compact set. By \cref{assumption:continuousmap} and Corollary~\ref{coro:esw_semicontinuous_2}, the mapping $\theta \mapsto \expeYLigne{ \swassersteinD[p](\hat{\mu}_n(\omega), \hat{\mu}_{\theta, m(n)}) }$ is \lsc, so there exists $\theta_n = \theta_n(\omega) \in \msk$ such that for any $n \in \mathbb{N}$, $\inf_{\theta \in \msk} \expeY{ \swassersteinD[p](\hat{\mu}_n(\omega), \hat{\mu}_{\theta, m(n)}) } = \expeY{ \swassersteinD[p](\hat{\mu}_n(\omega), \hat{\mu}_{\theta_n, m(n)}) }$. 

We consider the subsequence $\sequencen{\hat{\mu}_{\phi(n)}}$ where $\phi : \mathbb{N} \rightarrow \mathbb{N}$ is increasing such that $\expeYLigne{ \swassersteinD[p](\hat{\mu}_{\phi(n)}(\omega), \hat{\mu}_{\theta_{\phi(n)}, m(\phi(n))}) }$ converges to $\liminf_{n \rightarrow \infty} \expeYLigne{ \swassersteinD[p](\hat{\mu}_n(\omega), \hat{\mu}_{\theta_{n, m(n)}}) } = \liminf_{n \rightarrow \infty} \inf_{\theta \in \msk} \expeYLigne{ \swassersteinD[p](\hat{\mu}_n(\omega), \hat{\mu}_{\theta, m(n)}) }$. Since $\msk$ is compact, there also exists an increasing function $\psi : \mathbb{N} \rightarrow \mathbb{N}$ such that, for $\bar{\theta} \in \msk$, $\lim_{n \rightarrow \infty} \rho_\Theta(\theta_{\psi(\phi(n))}, \bar{\theta}) = 0$. Therefore, we have:  
    \begin{align}
        &\liminf_{n \rightarrow \infty} \inf_{\theta \in \msk} \expeY{ \swassersteinD[p](\hat{\mu}_n(\omega), \hat{\mu}_{\theta, m(n)}) } \nonumber \\
        &= \lim_{n \rightarrow \infty} \expeY{ \swassersteinD[p](\hat{\mu}_{\phi(n)}(\omega), \hat{\mu}_{\theta_{\phi(n)}, m(\phi(n))}) } \nonumber \\
        &= \lim_{n \rightarrow \infty} \expeY{ \swassersteinD[p](\hat{\mu}_{\psi(\phi(n))}(\omega), \hat{\mu}_{\theta_{\psi(\phi(n))}, m(\psi(\phi(n)))}) } \nonumber \\
        &= \liminf_{n \rightarrow \infty} \expeY{ \swassersteinD[p](\hat{\mu}_{\psi(\phi(n))}(\omega), \hat{\mu}_{\theta_{\psi(\phi(n))}, m(\psi(\phi(n)))}) } \nonumber \\
        &\geq \liminf_{n \rightarrow \infty} \defEns{ \swassersteinD[p](\hat{\mu}_{\psi(\phi(n))}(\omega), \mu_{\theta_{\psi(\phi(n))}}) - \expeY{ \swassersteinD[p](\mu_{\theta_{\psi(\phi(n))}}, \hat{\mu}_{\theta_{\psi(\phi(n))}, m(\psi(\phi(n)))})} } \label{eq:consistency_meswe_l1} \\
        &\geq \liminf_{n \rightarrow \infty} \swassersteinD[p](\hat{\mu}_{\psi(\phi(n))}(\omega), \mu_{\theta_{\psi(\phi(n))}}) - \limsup_{n \rightarrow \infty} \expeY{ \swassersteinD[p](\mu_{\theta_{\psi(\phi(n))}}, \hat{\mu}_{\theta_{\psi(\phi(n))}, m(\psi(\phi(n)))}) } \nonumber \\
        &\geq \swassersteinD[p](\mu_\star, \mu_{\bar{\theta}}) \label{eq:consistency_meswe_l2} \\
        &\geq \inf_{\theta \in \msk} \swassersteinD[p](\mu_\star, \mu_\theta) \nonumber
    \end{align}

    where \eqref{eq:consistency_meswe_l1} follows from the triangle inequality, and \eqref{eq:consistency_meswe_l2} is obtained on one hand by lower semi-continuity since $\hat{\mu}_{\psi(\phi(n))}(\omega) \wc \mu_\star$ by \cref{assumption:datagen} and Theorem~\ref{thm:SWp_metrizes_Pp} and $\mu_{\theta_{\psi(\phi(n))}} \wc \mu_{\bar{\theta}}$ by \cref{assumption:continuousmap}, and on the other hand by \cref{assumption:sw32} which gives $\limsup_{n \rightarrow \infty} \expeYLigne{ \swassersteinD[p](\mu_{\theta_{\psi(\phi(n))}}, \hat{\mu}_{\theta_{\psi(\phi(n))}, m(\psi(\phi(n)))}) } = 0$. We conclude that the first condition in \eqref{eqn:meswe_epiconv_cond} holds.

    Now, we fix $\mso \subset \Theta$ open. By definition of the infimum, there exists a sequence $\sequencen{\theta_n}$ in $\mso$ such that $\swassersteinD[p](\mu_\star, \mu_{\theta_n})$ converges to $\inf_{\theta \in \mso} \swassersteinD[p](\mu_\star, \mu_\theta)$. For any $n \in \mathbb{N}$, $\inf_{\theta \in \mso} \expeY{ \swassersteinD[p](\hat{\mu}_n(\omega), \hat{\mu}_{\theta, m(n)}) } \leq \expeY{ \swassersteinD[p](\hat{\mu}_n(\omega), \hat{\mu}_{\theta_n, m(n)}) }$. Therefore,
    \begin{align*}
        &\limsup_{n \rightarrow \infty} \inf_{\theta \in \mso} \expeY{ \swassersteinD[p](\hat{\mu}_n(\omega), \hat{\mu}_{\theta, m(n)}) } \leq \limsup_{n \rightarrow \infty} \expeY{ \swassersteinD[p](\hat{\mu}_n(\omega), \hat{\mu}_{\theta_n, m(n)}) } \nonumber \\
        & \qquad \qquad \qquad \leq \limsup_{n \rightarrow \infty} \defEns{ \swassersteinD[p](\hat{\mu}_n(\omega), \mu_\star) + \swassersteinD[p](\mu_\star, \mu_{\theta_n}) + \expeY{ \swassersteinD[p](\mu_{\theta_n}, \hat{\mu}_{\theta_n, m(n)}) } } \\
      & \qquad \qquad \qquad\qquad\qquad \text{by the triangle inequality} \\
        &\qquad \qquad \qquad = \limsup_{n \rightarrow \infty} \swassersteinD[p](\mu_\star, \mu_{\theta_n}) \;\; \text{ by \cref{assumption:datagen} and \cref{assumption:sw32}} \nonumber \\
        &\qquad \qquad \qquad = \inf_{\theta \in \mso} \swassersteinD[p](\mu_\star, \mu_\theta) \;\; \text{by definition of $\sequencen{\theta_n}$.}
    \end{align*}
    This shows that the second condition in \eqref{eqn:meswe_epiconv_cond} holds, and hence, the sequence of functions $\theta \mapsto \expeY{ \swassersteinD[p](\hat{\mu}_n(\omega), \hat{\mu}_{\theta, m(n)}) }$ epi-converges to $\theta \mapsto \swassersteinD[p](\mu_\star, \mu_\theta)$. 

    Now, we apply Theorem 7.31 of \cite{Rockafellar2009}. First, by \cite[Theorem 7.31(b)]{Rockafellar2009}, \eqref{eqn:meswe_consist_b} immediately follows from the epi-convergence of $\theta \mapsto \expeY{ \swassersteinD[p](\hat{\mu}_n(\omega), \hat{\mu}_{\theta, m(n)}) }$ to $\theta \mapsto \swassersteinD[p](\mu_\star, \mu_\theta)$.

Next, we show that \cite[Theorem 7.31(a)]{Rockafellar2009} holds by finding, for any $\eta > 0$, a compact set $\msb \subset \Theta$ and $N \in \mathbb{N}$ such that, for all $n \geq N$, 
\begin{equation*}
  \inf_{\theta \in \msb} \expeY{ \swassersteinD[p](\hat{\mu}_n(\omega), \hat{\mu}_{\theta, m(n)}) } \leq \inf_{\theta \in \Theta} \expeY{ \swassersteinD[p](\hat{\mu}_n(\omega), \hat{\mu}_{\theta, m(n)}) } + \eta \eqsp.
\end{equation*}

In fact, we simply show that there exists a compact set $\msb \subset \Theta$ and $N \in \mathbb{N}$ such that, for all $n \geq N$, $\inf_{\theta \in \msb} \expeY{ \swassersteinD[p](\hat{\mu}_n(\omega), \hat{\mu}_{\theta, m(n)}) } = \inf_{\theta \in \Theta} \expeY{ \swassersteinD[p](\hat{\mu}_n(\omega), \hat{\mu}_{\theta, m(n)}) }$. 

On one hand, the second condition in \eqref{eqn:meswe_epiconv_cond} gives us 
\begin{equation*}
  \limsup_{n \rightarrow \infty} \inf_{\theta \in \Theta} \expeY{ \swassersteinD[p](\hat{\mu}_n(\omega), \hat{\mu}_{\theta, m(n)}) } \leq \inf_{\theta \in \Theta} \swassersteinD[p](\mu_\star, \mu_\theta) = \epsilon_\star \eqsp.
\end{equation*} 

We deduce that there exists $n_{\epsilon / 6}(\omega)$ such that, for $n \geq n_{\epsilon / 6}(\omega)$, 
\begin{equation*}
  \inf_{\theta \in \Theta} \expeY{ \swassersteinD[p](\hat{\mu}_n(\omega), \hat{\mu}_{\theta, m(n)}) } \leq \epsilon_\star + \frac{\epsilon}{6},
\end{equation*}
with the $\epsilon$ of \cref{assumption:boundedset}. When $n \geq n_{\epsilon / 6}(\omega)$, the set $\widehat{\Theta}_{\epsilon/3} = \{ \theta \in \Theta : \expeYLigne{ \swassersteinD[p](\hat{\mu}_n(\omega), \hat{\mu}_{\theta, m(n)}) } \leq \epsilon_\star + \frac{\epsilon}{3} \}$ is non-empty as it contains $\theta^*$ defined as $\expeY{ \swassersteinD[p](\hat{\mu}_n(\omega), \hat{\mu}_{\theta^*, m(n)}) } = \inf_{\theta \in \Theta} \expeY{ \swassersteinD[p](\hat{\mu}_n(\omega), \hat{\mu}_{\theta, m(n)}) }$. 

On the other hand, by \cref{assumption:datagen}, there exists $n_{\epsilon/3}(\omega)$ such that, for $n \geq n_{\epsilon/3}(\omega)$, 
\begin{equation} \label{eqn:n_eps_3}
  \swassersteinD[p](\hat{\mu}_n(\omega), \mu_\star) \leq \frac{\epsilon}{3} \eqsp.
\end{equation}

Finally, by \cref{assumption:sw32}, there exists $n'_{\epsilon/3}(\omega)$ such that, for $n \geq n'_{\epsilon/3}(\omega)$, 
\begin{equation} \label{eqn:n_prime_eps_3}
  \expeY{ \swassersteinD[p](\mu_\theta, \hat{\mu}_{\theta, m(n)}) } \leq \frac{\epsilon}{3} \eqsp.
\end{equation}

Let $n \geq n_*(\omega) = \max\{ n_{\epsilon/6}(\omega), n_{\epsilon/3}(\omega), n'_{\epsilon/3}(\omega) \}$ and $\theta \in \widehat{\Theta}_{\epsilon/3}$. By the triangle inequality,
\begin{align*}
  \swassersteinD[p](\mu_\star, \mu_\theta) &\leq \swassersteinD[p](\hat{\mu}_n(\omega), \mu_\star) + \expeY{ \swassersteinD[p](\hat{\mu}_n(\omega), \hat{\mu}_{\theta, m(n)}) } + \expeY{ \swassersteinD[p](\mu_\theta, \hat{\mu}_{\theta, m(n)}) } \\{}
  &\leq \epsilon_\star + \epsilon \text{\;\;\; since } \theta \in \widehat{\Theta}_{\epsilon/3} \text{ and by } \eqref{eqn:n_eps_3} \text{ and } \eqref{eqn:n_prime_eps_3} 
\end{align*}

This means that, when $n \geq n_*(\omega)$, $\widehat{\Theta}_{\epsilon/3} \subset \Theta^\star_\epsilon$ with $\Theta^\star_\epsilon$ as defined in \cref{assumption:boundedset}, and since $\inf_{\theta \in \Theta} \expeY{ \swassersteinD[p](\hat{\mu}_n(\omega), \hat{\mu}_{\theta, m(n)}) }$ is attained in $\widehat{\Theta}_{\epsilon/3}$, we have
\begin{align}
  \inf_{\theta \in \Theta^\star_{\epsilon}} \expeY{ \swassersteinD[p](\hat{\mu}_n(\omega), \hat{\mu}_{\theta, m(n)}) } &= \inf_{\theta \in \Theta} \expeY{ \swassersteinD[p](\hat{\mu}_n(\omega), \hat{\mu}_{\theta, m(n)}) } \eqsp. \label{eqn:meswe_cond_for_thm731a}
\end{align}

By \cite[Theorem 7.31(a)]{Rockafellar2009}, \eqref{eqn:meswe_consist_a} is a direct consequence of \eqref{eqn:meswe_cond_for_thm731a} and the epi-convergence of $\theta \mapsto \expeY{ \swassersteinD[p](\hat{\mu}_n(\omega), \hat{\mu}_{\theta, m(n)}) }$ to $\theta \mapsto \swassersteinD[p](\mu_\star, \mu_\theta)$.

Finally, by the same reasoning that was done earlier in this proof for $\argmin_{\theta \in \Theta} \swassersteinD[p](\mu_\star, \mu_\theta)$, the set $\argmin_{\theta \in \Theta} \expeY{ \swassersteinD[p](\hat{\mu}_n(\omega), \hat{\mu}_{\theta, m(n)}) }$ is non-empty for $n \geq n_*(\omega)$. 

\end{proof}

\subsection{Convergence of the MESWE to the MSWE: Proof of \Cref{thm:cvg_meswe_to_mswe}}

\begin{proof}[Proof of \Cref{thm:cvg_meswe_to_mswe}]
  Here again, the result follows from applying \cite[Theorem 7.31]{Rockafellar2009}, paraphrased in Theorem~\ref{thm:731rockafellar}.
    
    First, by \cref{assumption:continuousmap} and Corollary~\ref{coro:sw_semicontinuous_2}, the map $\theta \mapsto \swassersteinD[p](\hat{\mu}_n, \mu_\theta)$ is l.s.c. on $\Theta$. Therefore, there exists $\theta_n \in \Theta$ such that $\swassersteinD[p](\hat{\mu}_n, \mu_{\theta_n}) = \epsilon_n$. The set $\Theta_{\epsilon, n}$ with the $\epsilon$ from \cref{assumption:boundedset_eps_n} is non-empty as it contains $\theta_n$, closed by lower semi-continuity of $\theta \mapsto \swassersteinD[p](\hat{\mu}_n, \mu_\theta)$, and bounded. $\Theta_{\epsilon, n}$ is thus compact, and we conclude again by lower semi-continuity that the set $\argmin_{\theta \in \Theta} \swassersteinD[p](\hat{\mu}_n, \mu_\theta)$ is non-empty \cite[Theorem 2.43]{aliprantis1999infinite}.
    
    Then, we prove that $\theta \mapsto \expeY{ \swassersteinD[p](\hat{\mu}_n, \hat{\mu}_{\theta,m}) }$ epi-converges to $\theta \mapsto \swassersteinD[p](\hat{\mu}_n, \mu_\theta)$ as $m \rightarrow \infty$ using the characterization in \cite[Proposition 7.29]{Rockafellar2009}, \ie~we verify that: for every compact set $\msk \subset \Theta$ and every open set $\mso \subset \Theta$,
    \begin{equation}
      \begin{aligned}
        \liminf_{m \rightarrow \infty} \inf_{\theta \in \msk} \expeY{ \swassersteinD[p](\hat{\mu}_n, \hat{\mu}_{\theta, m})}& \geq \inf_{\theta \in \msk} \swassersteinD[p](\hat{\mu}_n, \mu_\theta)   \\
        \limsup_{m \rightarrow \infty} \inf_{\theta \in \mso} \expeY{ \swassersteinD[p](\hat{\mu}_n, \hat{\mu}_{\theta, m})}& \leq \inf_{\theta \in \mso} \swassersteinD[p](\hat{\mu}_n, \mu_\theta) \eqsp. 
      \end{aligned}
      \label{eqn:meswe_to_mswe_epiconv}
    \end{equation}
    
    Let $\msk \subset \Theta$ be a compact set. By \cref{assumption:continuousmap} and Corollary~\ref{coro:esw_semicontinuous_2}, for any $m \in \mathbb{N}$, the map $\theta \mapsto \expeYLigne{ \swassersteinD[p](\hat{\mu}_n, \hat{\mu}_{\theta, m}) }$ is l.s.c., so there exists $\theta_m \in \msk$ such that $\inf_{\theta \in \msk} \expeYLigne{ \swassersteinD[p](\hat{\mu}_n, \hat{\mu}_{\theta, m}) } = \expeYLigne{ \swassersteinD[p](\hat{\mu}_n, \hat{\mu}_{\theta_m, m}) }$. 

  We consider the subsequence $\{ \hat{\mu}_{\theta_{\phi(m)}, \phi(m)} \}_{m \in \mathbb{N}}$ where $\phi : \mathbb{N} \rightarrow \mathbb{N}$ is increasing such that $\expeYLigne{ \swassersteinD[p](\hat{\mu}_n, \hat{\mu}_{\theta_{\phi(m)}, \phi(m)}) }$ converges to $\liminf_{m \rightarrow \infty} \ \expeYLigne{ \swassersteinD[p](\hat{\mu}_n, \hat{\mu}_{\theta_m, m}) } = \liminf_{m \rightarrow \infty} \inf_{\theta \in \msk}  \expeYLigne{ \swassersteinD[p](\hat{\mu}_n, \hat{\mu}_{\theta,m}) }$. Since $\msk$ is compact, there also exists an increasing function $\psi : \mathbb{N} \rightarrow \mathbb{N}$ such that, for any $\bar{\theta} \in \msk$, $\lim_{m \rightarrow \infty } \rho_\Theta(\theta_{\psi(\phi(m))}, \bar{\theta}) = 0$. Therefore, we have  
    \begin{align}
        &\liminf_{m \rightarrow \infty} \inf_{\theta \in \msk} \expeY{ \swassersteinD[p](\hat{\mu}_n, \hat{\mu}_{\theta, m}) } \nonumber \\
        &= \lim_{m \rightarrow \infty} \expeY{ \swassersteinD[p](\hat{\mu}_n, \hat{\mu}_{\theta_{\phi(m)}, \phi(m)}) } \nonumber \\
        &= \lim_{m \rightarrow \infty} \expeY{ \swassersteinD[p](\hat{\mu}_n, \hat{\mu}_{\theta_{\psi(\phi(m))}, \psi(\phi(m))}) } \nonumber \\
        &= \liminf_{m \rightarrow \infty} \expeY{ \swassersteinD[p](\hat{\mu}_n, \hat{\mu}_{\theta_{\psi(\phi(m))}, \psi(\phi(m))}) } \nonumber \\
        &\geq \liminf_{m \rightarrow \infty} [ \swassersteinD[p](\hat{\mu}_n, \mu_{\theta_{\psi(\phi(m))}}) - \expeY{ \swassersteinD[p](\mu_{\theta_{\psi(\phi(m))}}, \hat{\mu}_{\theta_{\psi(\phi(m))}, \psi(\phi(m))}) } ] \label{eqn:proof_meswe_to_mswe_l1} \\
        &\geq \liminf_{m \rightarrow \infty} \swassersteinD[p](\hat{\mu}_n, \mu_{\theta_{\psi(\phi(m))}}) - \limsup_{m \rightarrow \infty} \expeY{ \swassersteinD[p](\mu_{\theta_{\psi(\phi(m))}}, \hat{\mu}_{\theta_{\psi(\phi(m))}, \psi(\phi(m))}) } \nonumber \\
        &\geq \swassersteinD[p](\hat{\mu}_n, \mu_{\bar{\theta}}) \label{eqn:proof_meswe_to_mswe_l2} \\
        &\geq \inf_{\theta \in \msk} \swassersteinD[p](\hat{\mu}_n, \mu_\theta) \nonumber
    \end{align}

    where \eqref{eqn:proof_meswe_to_mswe_l1} results from the triangle inequality and \eqref{eqn:proof_meswe_to_mswe_l2} is obtained by \cref{assumption:sw32} on one hand and by lower semi-continuity on the other hand since $\mu_{\theta_{\psi(\phi(n))}} \wc \mu_{\bar{\theta}}$ by \cref{assumption:continuousmap}. We conclude that the first condition in \eqref{eqn:meswe_to_mswe_epiconv} holds.
    
    Now, we fix $\mso \subset \Theta$ open. By definition of the infimum, there exists a sequence $(\theta_m)_{m \in \mathbb{N}}$ in $\mso$ such that $\swassersteinD[p](\hat{\mu}_n, \hat{\mu}_{\theta_m, m})$ converges to $\inf_{\theta \in \mso} \swassersteinD[p](\hat{\mu}_n, \hat{\mu}_{\theta,m})$. For any $m \in \mathbb{N}$, $\inf_{\theta \in \mso} \expeY{ \swassersteinD[p](\hat{\mu}_n, \hat{\mu}_{\theta, m}) } \leq \expeY{ \swassersteinD[p](\hat{\mu}_n, \mu_{\theta_m, m}) }$. Therefore,
    \begin{align*}
        &\limsup_{m \rightarrow \infty} \inf_{\theta \in \mso} \expeY{ \swassersteinD[p](\hat{\mu}_n, \hat{\mu}_{\theta, m}) } \\
        &\leq \limsup_{m \rightarrow \infty} \expeY{ \swassersteinD[p](\hat{\mu}_n, \hat{\mu}_{\theta_m, m}) } \nonumber \\
        &\leq \limsup_{m \rightarrow \infty} [ \swassersteinD[p](\hat{\mu}_n, \mu_{\theta_m}) + \expeY{ \swassersteinD[p](\mu_{\theta_m}, \hat{\mu}_{\theta_m, m}) } ] \;\; \text{by the triangle inequality} \\
        &\leq  \limsup_{m \rightarrow \infty} \swassersteinD[p](\hat{\mu}_n, \mu_{\theta_m}) \text{ by \cref{assumption:sw32}} \\
        &= \inf_{\theta \in \mso} \swassersteinD[p](\hat{\mu}_n, \mu_\theta) \;\; \text{by definition of $(\theta_m)_{m \in \mathbb{N}}$}
    \end{align*}
    This shows that the second condition in \eqref{eqn:meswe_to_mswe_epiconv} holds, and hence, the sequence of functions $\theta \mapsto \expeY{ \swassersteinD[p](\hat{\mu}_n, \hat{\mu}_{\theta,m}) }$ epi-converges to $\theta \mapsto \swassersteinD[p](\hat{\mu}_n, \mu_\theta)$. 

    Now, we apply \cite[Theorem 7.31]{Rockafellar2009}. By \cite[Theorem 7.31(b)]{Rockafellar2009}, \eqref{eqn:meswe_to_mswe_b} immediately follows from the epi-convergence of $\theta \mapsto \expeY{ \swassersteinD[p](\hat{\mu}_n, \hat{\mu}_{\theta,m}) }$ to $\theta \mapsto \swassersteinD[p](\hat{\mu}_n, \mu_\theta)$.

    Next, we show that \cite[Theorem 7.31(a)]{Rockafellar2009} holds by finding for any $\eta > 0$ a compact set $\msb \subset \Theta$ and $N \in \nset$ such that, for all $n \geq N$, 
  \begin{equation*}
    \inf_{\theta \in \msb} \expeY{ \swassersteinD[p](\hat{\mu}_n, \hat{\mu}_{\theta, m}) } \leq \inf_{\theta \in \Theta} \expeY{ \swassersteinD[p](\hat{\mu}_n, \hat{\mu}_{\theta, m}) } + \eta \eqsp.
  \end{equation*}

  In fact, we simply show that there exists a compact set $\msb \subset \Theta$ and $N \in \nset$ such that, for all $n \geq N$, $\inf_{\theta \in \msb} \expeY{ \swassersteinD[p](\hat{\mu}_n, \hat{\mu}_{\theta, m}) } = \inf_{\theta \in \Theta} \expeY{ \swassersteinD[p](\hat{\mu}_n, \hat{\mu}_{\theta, m}) }$.
  On one hand, the second condition in \eqref{eqn:meswe_to_mswe_epiconv} gives us 
  \begin{equation*}
    \limsup_{m \rightarrow \infty} \inf_{\theta \in \Theta} \expeY{ \swassersteinD[p](\hat{\mu}_n, \hat{\mu}_{\theta, m}) } \leq \inf_{\theta \in \Theta} \swassersteinD[p](\hat{\mu}_n, \mu_\theta) = \epsilon_n \eqsp.
  \end{equation*} 

  We deduce that there exists $m_{\epsilon/4}$ such that, for $m \geq m_{\epsilon/4}$, 
  \begin{equation} \label{eqn:m_eps4}
    \inf_{\theta \in \Theta} \expeY{ \swassersteinD[p](\hat{\mu}_n, \hat{\mu}_{\theta, m}) } \leq \epsilon_n + \frac{\epsilon}{4} \eqsp.
  \end{equation}
  with the $\epsilon$ of \cref{assumption:boundedset_eps_n}. When $m \geq m_{\epsilon/4}$, the set $\Theta_{\epsilon/2} = \{ \theta \in \Theta : \expeYLigne{ \swassersteinD[p](\hat{\mu}_n, \hat{\mu}_{\theta, m}) } \leq \epsilon_n + \frac{\epsilon}{2} \}$ is non-empty as it contains $\theta^*$ defined as $\expeYLigne{ \swassersteinD[p](\hat{\mu}_n, \hat{\mu}_{\theta^*, m}) } = \inf_{\theta \in \Theta} \expeYLigne{ \swassersteinD[p](\hat{\mu}_n, \hat{\mu}_{\theta, m}) }$. 

  On the other hand, by \cref{assumption:sw32}, there exists $m_{\epsilon/2}$ such that, for $m \geq m_{\epsilon/2}$, 
  \begin{equation} \label{eqn:m_eps2}
    \expeY{ \swassersteinD[p](\mu_\theta, \hat{\mu}_{\theta, m}) } \leq \frac{\epsilon}{2} \eqsp.
  \end{equation}

  Let $\theta$ belong to $\Theta_{\epsilon/2}$ and $m \geq m_* = \max\{m_{\epsilon/4}, m_{\epsilon/2} \}$. By the triangle inequality,
  \begin{align*}
    \swassersteinD[p](\hat{\mu}_n, \mu_\theta) &\leq \expeY{ \swassersteinD[p](\hat{\mu}_n, \hat{\mu}_{\theta, m}) } + \expeY{ \swassersteinD[p](\mu_\theta, \hat{\mu}_{\theta, m}) } \\
    &\leq \epsilon_n + \epsilon \text{\;\;\; since $\theta \in \Theta_{\epsilon/2}$ and by \eqref{eqn:m_eps2}}  
  \end{align*}

  This means that, when $m \geq m_*$, $\Theta_{\epsilon/2} \subset \Theta_{\epsilon, n}$, and since $\inf_{\theta \in \Theta} \expeY{ \swassersteinD[p](\hat{\mu}_n, \hat{\mu}_{\theta,m}) }$ is attained in $\Theta_{\epsilon/2}$,
  \begin{equation}
    \inf_{\theta \in \Theta_{\epsilon, n}} \expeY{ \swassersteinD[p](\hat{\mu}_n, \hat{\mu}_{\theta,m}) } = \inf_{\theta \in \Theta} \expeY{ \swassersteinD[p](\hat{\mu}_n, \hat{\mu}_{\theta,m}) } \eqsp. \label{eqn:meswe_to_mswe_cond_for_thm731a}
  \end{equation}

  By \cite[Theorem 7.31(a)]{Rockafellar2009}, \eqref{eqn:meswe_to_mswe_a} is a direct consequence of \eqref{eqn:meswe_to_mswe_cond_for_thm731a} and the epiconvergence of $\theta \mapsto \expeY{ \swassersteinD[p](\hat{\mu}_n(\omega), \hat{\mu}_{\theta,m}) }$ to $\theta \mapsto \swassersteinD[p](\hat{\mu}_n, \mu_\theta)$.

  Finally, by the same reasoning that was done earlier in this proof for $\argmin_{\theta \in \Theta} \swassersteinD[p](\hat{\mu}_n, \mu_\theta)$, the set $\argmin_{\theta \in \Theta} \expeY{ \swassersteinD[p](\hat{\mu}_n, \hat{\mu}_{\theta,m}) }$ is non-empty for $m \geq m_*$. 

\end{proof}

\subsection{Proof of Rate of convergence and asymptotic distribution: Proof of \Cref{thm:asymptotic_1} and \Cref{thm:asymptotic_2}}

\begin{proof}[Proof of \Cref{thm:asymptotic_1} and \Cref{thm:asymptotic_2}]
  The proof of \Cref{thm:asymptotic_1} and \Cref{thm:asymptotic_2} consists in showing that the conditions of Theorem 4.2 and Theorem 7.2 in \cite{Pollard1980} respectively are satisfied: conditions (i), (ii) and (iii) follow from \cref{assumption:well_separation}, \cref{assumption:form_derivative} and \cref{assumption:weak_convergence_without_norm}.
\end{proof}

\section{Computational Aspects} \label{appendix:computational}

The MSWE and MESWE are in general computationnally intractable, partly because the Sliced-Wasserstein distance requires an integration over infinitely many projections. In this section, we review the numerical methods used to approximate these two estimators.

{\bf Approximation of $\swassersteinD[p]$:} We recall the definition of the SW distance below.
\begin{equation}
\label{eq:def_sliced_wasser_supp}
  \swassersteinD[p]^{p}(\mu, \nu) = \int_{\sphere^{d-1}} \wassersteinD[p]^p(\uss_{\sharp} \mu, \uss_{\sharp} \nu) \rmd\unifS(\us) \eqsp,
\end{equation}
where $\unifS$ is the uniform distribution on $\sphere^{d-1}$ and for any measurable function $f :\msy \to \rset$ and $\zeta \in \mcp(\msy)$,  $f_{\sharp}\zeta$ is the push-forward measure of $\zeta$ by $f$. We approximate the integral in \eqref{eq:def_sliced_wasser_supp} by selecting a finite set of projections $\msu \subset \sphere^{d-1}$ and computing the empirical average:
\begin{equation} \label{eq:def_approx_sliced_wasser}
  \swassersteinD[p]^{p}(\mu, \nu) \approx \frac{1}{\card(\msu)}\sum_{u \in \msu} \wassersteinD[p]^p(\uss_{\sharp} \mu, \uss_{\sharp} \nu)
\end{equation}

The quality of this approximation depends on the sampling of $\sphereD$. In our work, we use random samples picked uniformly on $\sphereD$, as proposed in \cite{bonneel2015} and explained hereafter (see paragraph ``Sampling schemes'').

The Wasserstein distance between two one-dimensional probability densities $\mu$ and $\nu$  as defined in \eqref{eq:wp1d} is also estimated by replacing the integrals with a Monte Carlo estimate, and we can use two distinct methods to approximate this quantity. %

The first approximation we consider is given by,
\begin{equation}
\label{eq:wass_approx_1}
      \wassersteinD[p]^p(\mu, \nu) \approx \frac{1}{K} \sum_{k=1}^K \abs{ \tilde{F}_\mu^{-1}(t_k) - \tilde{F}_\nu^{-1}(t_k)}^p \eqsp,
    \end{equation}
    where $\{t_k\}_{k=1}^K$ are uniform and independent samples from $\ccint{0,1}$ and  for $\xi \in \{\mu,\nu\}$, $\tilde{F}_\xi^{-1}$ is a linear interpolation of $\bar{F}^{-1}_{\xi}$ which denotes either the exact quantile function of $\xi$ if $\xi$ is discrete, or an approximation by a Monte Carlo procedure. This last option is justified by the Glivenko-Cantelli Theorem.

The second approximation is given by,
\begin{equation}
\label{eq:wass_approx_2}
      \wassersteinD[p]^p(\mu, \nu) \approx \frac{1}{K} \sum_{k=1}^K \abs{ s_k - \tilde{F}_\nu^{-1}(\tilde{F}_\mu(s_k))}^p \eqsp,
\end{equation}
    where $\{s_k\}_{i=1}^K$ are uniform and independent samples from $\mu$  and  for $\xi \in \{\mu,\nu\}$, $\tilde{F}_\xi$ (resp. $\tilde{F}_\xi^{-1}$)  is a linear interpolation of $\bar{F}_{\xi}$ (resp. $\bar{F}^{-1}_{\xi}$) which denotes either the exact cumulative distribution function (resp. quantile function) of $\xi$ if $\xi$ is discrete or an approximation by a Monte Carlo procedure. 

{\bf Sampling schemes:} We explain the methods that we used to generate \iid~samples from the uniform distribution on the $d$-dimensional sphere $\sphereD$ and from multivariate elliptically contoured stable distributions. 
\begin{itemize}
  \item \textbf{Uniform sampling on the sphere.} To sample from $\sphereD$, we form the $d$-dimensional vector $\bfs$ by drawing each of its $d$ components from the standard normal distribution $\calN(0, 1)$ and we normalize it: $\mathbf{s'} = \bfs / \| \bfs \|_2$, so that $\mathbf{s'}$ lies on the sphere.  
  \item \textbf{Sampling from multivariate elliptically contoured stable distributions.} We recall that if $Y \in \Rd$ is $\alpha$-stable and elliptically contoured, \ie~$Y \sim \ellstable(\bfSigma, \bfm)$, then its joint characteristic function is defined as, for any $\bft \in \Rd$,
  \begin{equation} \label{eq:ellstable_charfunc}
    \Esp [ \exp (i\bft^T Y) ] = \exp \left( - (\bft^T \bfSigma \bft)^{\alpha / 2} + i \bft^T \bfm \right) \eqsp ,
  \end{equation}
  where $\bfSigma$ is a positive definite matrix (akin to a correlation matrix), $\bfm \in \Rd$ is a location vector (equal to the mean if it exists) and $\alpha \in (0, 2)$ controls the thickness of the tail. Elliptically contoured stable distributions are scale mixtures of multivariate Gaussian distributions \cite[Proposition 2.5.2]{samorodnitsky1994stable}, whose densities are intractable, but can easily be simulated \cite{Nolan2013}: let $A \sim \calS_{\alpha/2}(\beta, \gamma, \delta)$ be a one-dimensional positive $(\alpha / 2)$-stable random variable with $\beta = 1$, $\gamma = 2\cos(\frac{\pi \alpha}{4})^{2/\alpha}$ and $\delta = 0$, and $G \sim \calN(\bfnot, \bfSigma)$. Then, $Y = \sqrt{A}G + \bfm$ has \eqref{eq:ellstable_charfunc} as characteristic function. 
\end{itemize}

{\bf Optimization methods:} Computing the MSWE and MESWE implies minimizing the (expected) Sliced-Wasserstein distance over the set of parameters. In our experiments, we used different optimization methods as we detail below.
\begin{itemize}

\item \textbf{Multivariate Gaussian distributions.} We derive the explicit gradient expressions of the approximate $\swassersteinD[2]^2$ distance with respect to the mean and scale parameters $\bfm$ and $\sigma^2$, and we use the ADAM stochastic optimization method with the default parameter settings suggested in \cite{Adam}. For the MSWE, we use \eqref{eq:wass_approx_2} to approximate the one-dimensional Wasserstein distance, and we evaluate directly the Gaussian density of the generated samples, utilizing the fact that the projection of a Gaussian of parameters $(\bfm, \sigma^2 \bfI)$ along $u \in \sphereD$ is a 1D normal distribution of parameters $(\langle u, \bfm \rangle, \sigma^2 \langle u, u \rangle)$. In this case, the gradient of the approximate $\swassersteinD[2]^2$ between $\mu = \calN(\bfm, \sigma^2\bfI)$ and the empirical distribution associated to $n$ samples drawn by $\calN(\bfm_\star, \sigma^2_\star\bfI)$, denoted by $\hat{\nu}$, is given by,
\begin{align*}
  \nabla_\bfm \swassersteinD[2]^2(\mu, \hat{\nu}) = \frac{1}{\card(\msu)\card(\mss)} \sum_{u \in \msu, s \in \mss} \bigg( \abs{ s - \tilde{F}_{\usss \hat{\nu}}^{-1}(\tilde{F}_{\usss \mu}(s))}^2 \calN(s ; \langle u, \bfm \rangle, \sigma^2 \norm{u}^2)& \\
  \frac{s - \langle u, \bfm \rangle}{\sigma^2 \norm{u}^2} u \bigg),& \\
  \nabla_{\sigma^2} \swassersteinD[2]^2(\mu, \hat{\nu}) = \frac{1}{\card(\msu)\card(\mss)} \sum_{u \in \msu, s \in \mss} \bigg( \abs{ s - \tilde{F}_{\usss \hat{\nu}}^{-1}(\tilde{F}_{\usss \mu}(s))}^2 \calN(s ; \langle u, \bfm \rangle, \sigma^2 \norm{u}^2) \\
   \frac{1}{2 \sigma^2 }\left(\frac{(s - \langle u, \bfm \rangle)^2}{\sigma^2 \norm{u}^2} - 1 \right) \bigg),&
\end{align*}
where $\msu \subset \sphereD$ is a finite set of random projections picked uniformly on $\sphereD$, $\mss$ is a finite subset in $\rset$, and for any $s \in \mss$, $\calN(s ; \langle u, \bfm \rangle, \sigma^2 \norm{u}^2)$ denotes the density function of the Gaussian of parameters $(\langle u, \bfm \rangle, \sigma^2 \norm{u}^2)$ evaluated at $s$. \\

For the MESWE, we use \eqref{eq:wass_approx_1} and evaluate the empirical distribution of generated samples instead of their normal density. Therefore, the gradient of the approximate $\swassersteinD[2]^2$ between the empirical distributions corresponding to one generated dataset of $m$ samples drawn from $\calN(\mu, \sigma^2\bfI)$ and $n$ samples drawn from $\calN(\mu_\star, \sigma^2_\star\bfI)$, respectively denoted by $\hat{\mu}$ and $\hat{\nu}$, is obtained with,
\begin{align*}
  \nabla_\bfm \swassersteinD[2]^2(\hat{\mu}, \hat{\nu}) &= \frac{-2}{\card(\msu).K} \sum_{u \in \msu} \sum_{k=1}^K \abs{ \tilde{F}_{\usss \hat{\mu}}^{-1}(t_k) - \tilde{F}_{\usss \hat{\nu}}^{-1}(t_k)} u \eqsp , \\
  \nabla_{\sigma^2} \swassersteinD[2]^2(\hat{\mu}, \hat{\nu}) &= \frac{1}{\card(\msu).K} \sum_{u \in \msu} \sum_{k=1}^K \abs{ \tilde{F}_{\usss \hat{\mu}}^{-1}(t_k) - \tilde{F}_{\usss \hat{\nu}}^{-1}(t_k)} \frac{\langle u, \bfm \rangle - \tilde{F}_{\usss \hat{\mu}}^{-1}(t_k)}{\sigma^2} \eqsp .
\end{align*}

\item \textbf{Multivariate elliptically contoured stable distributions.} When comparing MESWE to MEWE, we approximate these estimators using the derivative-free optimization method Nelder-Mead (implemented in \texttt{Scipy}), following the approach in \cite{Bernton2019}. 

When illustrating the theoretical properties of MESWE, we proceed in the same way as for the multivariate Gaussian experiment: we compute the explicit gradient expression of the approximate $\swassersteinD[2]^2$ distance with respect to the location parameter $\bfm$, and we use the ADAM stochastic optimization method with the default settings. \Cref{eq:alpha_gradient} gives the formula of the gradient of the approximate $\swassersteinD[2]^2$ between the empirical distributions of one generated dataset of $m$ samples drawn from $\ellstable(\bfI, \bfm)$ and $n$ samples drawn from $\ellstable(\bfI, \bfm_\star)$, respectively denoted by $\hat{\mu}$ and $\hat{\nu}$, with respect to $\bfm$.
\begin{equation} \label{eq:alpha_gradient}
\nabla_\bfm \swassersteinD[2]^2(\hat{\mu}, \hat{\nu}) =  \frac{-2}{\card(\msu).K} \sum_{u \in \msu} \sum_{k=1}^K \abs{ \tilde{F}_{\usss \hat{\mu}}^{-1}(t_k) - \tilde{F}_{\usss \hat{\nu}}^{-1}(t_k)} u \eqsp .
\end{equation}

\item \textbf{High-dimensional real data using GANs.} We use the ADAM optimizer provided by TensorFlow GPU.
\end{itemize}

{\bf Computing infrastructure:} The experiment comparing the computational time of MESWE and MEWE was conducted on a daily-use laptop (CPU intel core i7, 1.90GHz $\times$ 8 and 16GB of RAM). The neural network experiment was run on a cluster with 4 relatively modern GPUs.

\end{document}